\documentclass{article}

\usepackage{arxiv}

\usepackage[utf8]{inputenc} 
\usepackage[T1]{fontenc}    
\usepackage{hyperref}       
\usepackage{url}            
\usepackage{booktabs}       
\usepackage{amsfonts}       
\usepackage{nicefrac}       
\usepackage{microtype}      
\usepackage{lipsum}
\usepackage{graphicx}
\graphicspath{ {./images/} }

\usepackage{amsmath,amsthm,bbm}
\usepackage{amsfonts,amssymb}
\usepackage{amsmath,amsxtra,amssymb,latexsym, amscd,amsthm}
\usepackage{indentfirst}
\usepackage{fancyhdr}
\usepackage{subfigure}
\usepackage{graphicx}
\usepackage[mathscr]{eucal}
\usepackage{mathrsfs}
\usepackage{makeidx}
\usepackage{mathabx}
\usepackage{multicol}
\usepackage{color}
\usepackage{multirow}
\usepackage{bm}
\usepackage{fancyhdr}
\usepackage[switch,columnwise]{lineno}
\usepackage{lipsum}
\usepackage{fancyhdr}
\usepackage{tabularx}
\usepackage{makecell,rotating,multirow,diagbox}
\usepackage{makecell}

\newcommand{\NN}{\mathbb{N}}
\newcommand{\RR}{\mathbb{R}}

\newtheorem{thm}{Theorem}
\newtheorem{cor}[thm]{Corollary}
\newtheorem{lem}[thm]{Lemma}

\newtheorem{remark}[thm]{Remark}
\newtheorem{defn}{Definition}

\newtheorem{assump}{Assumption}

\newcommand{\thmref}[1]{Theorem~\ref{#1}}
\newcommand{\lemref}[1]{Lemma~\ref{#1}}
\newcommand{\corref}[1]{Corollary~\ref{#1}}

\newcommand{\defref}[1]{Definition~\ref{#1}}
\newcommand{\assref}[1]{Assumption~\ref{#1}}

\def\f{\frac}

\def\D{\mathcal{D}}

\def\X{\mathcal{X}}

\def\S{\mathcal{S}}

\def\t{\mathcal{T}}

\def\RR{\mathbb{R}}

\DeclareMathOperator\supp{supp}

\def\b{\mathbf}
\def\rank{\operatorname{rank}}

\def\X{\mathcal{X}}
\def\sign{\text{sign}}

\title{Convergence Analysis for Deep Sparse Coding via Convolutional Neural Networks}

\author{
	{\bf Jianfei Li}\thanks{Department of Mathematics, Ludwig-Maximilians-Universit{\"a}t M{\"u}nchen
		(\texttt{lijianfei@math.lmu.de})}
	\and
	 {\bf Han Feng}\thanks{Department of Mathematics, City University of Hong Kong (\texttt{hanfeng@cityu.edu.hk})}
	\and
	{\bf Ding-Xuan Zhou}\thanks{Department of Mathematics, The University of Sydney
		(\texttt{dingxuan.zhou@sydney.edu.au})}
}

\date{}

\begin{document}
\maketitle

\begin{abstract}
In this work, we explore the intersection of sparse coding theory and deep learning to enhance our understanding of feature extraction capabilities in advanced neural network architectures. We begin by introducing a novel class of Deep Sparse Coding (DSC) models and establish a thorough theoretical analysis of their uniqueness and stability properties. By applying iterative algorithms to these DSC models, we derive convergence rates for convolutional neural networks (CNNs) in their ability to extract sparse features. This provides a strong theoretical foundation for the use of CNNs in sparse feature-learning tasks. We additionally extend this convergence analysis to more general neural network architectures, including those with diverse activation functions, as well as self-attention and transformer-based models. This broadens the applicability of our findings to a wide range of deep learning methods for the extraction of deep-sparse features. Inspired by the strong connection between sparse coding and CNNs, we also explore training strategies to encourage neural networks to learn sparser features. Through numerical experiments, we demonstrate the effectiveness of these approaches, providing valuable insight for the design of efficient and interpretable deep learning models.
\\
\par\

\end{abstract}

\section{Introduction}
Convolutional neural networks (CNNs) have achieved remarkable successes in a wide range of computer vision and signal processing tasks, demonstrating superior performances compared to traditional machine learning approaches. Theoretical insights to understand the ability of CNNs to learn effective feature representations from data are crucial to further advance their capabilities and enable their widespread adoption in various domains.

In the field of signal processing, the concept of sparse feature extraction has long been recognized as a powerful technique to efficiently represent complex data. Sparse coding models aim to decompose signals into a linear combination of a small number of basis elements, enabling compact and informative feature representations. These sparse feature extraction methods have been extensively studied and applied in areas such as image denoising, compressed sensing, and dictionary learning.

The classical sparse coding algorithm is originally designed to represent data $\b y\in \RR^m$ into a sparse linear combination of a given dictionary $\b D\in \RR^{m\times d}$ by solving the following constrained optimization problem:
\begin{equation*}
	\begin{array}{rrclcl}
		\displaystyle (P_0)\quad \min_{\b x} & \|\b x\|_0 \quad
		\textrm{s.t.} & \b y = \b D \b x,
	\end{array}
\end{equation*}
where $\b x:= (x_i)_{i=1}^d \in \RR^d$ and $\|\b x\|_0 := \#\{ i: x_i \neq 0 \}$.

Usually, the strict constraint of $(P_0)$ is relaxed with error tolerance taking into account the noise involved in observation $\b y$. Hence, instead, one may consider an error-tolerant version of $(P_0)$, by solving
\begin{equation*}
	\begin{array}{rrclcl}
		\displaystyle (P_0^\varepsilon)\quad \min_{\b x} & \|\b x\|_0 \quad
		\textrm{s.t.} & \|\b y - \b D \b x\|_2\leq \varepsilon,
	\end{array}
\end{equation*}
for $\varepsilon>0$. Note that $\|\b \cdot\|_0$ is too sensitive to small entries and is not convex.

Consequently, to address the non-convexity and sensitivity issues of the $\ell_0$-based formulation, researchers have proposed alternative approaches that relax the $\ell_0$ ``norm'' to the more tractable $\ell_1$ norm. This has led to the well-known basis pursuit denoising (BPDN) and least absolute shrinkage and selection operator (LASSO) problems, which can be formulated as the following variant with a regularization parameter $\lambda > 0$:
\begin{equation*}
	\begin{array}{rrclcl}
		\displaystyle (P_{1,\lambda})\quad \min_{\b x} & \|\b x\|_1 + \lambda \|\b y-\b D\b x\|^2_2.
	\end{array}
\end{equation*}
The $\ell_1$ regularization term encourages sparse solutions, while the data fidelity term ensures that the reconstructed signal $\b D\b x$ is close to the observed signal $\b y$. This formulation has led to the development of many efficient algorithms for solving sparse coding problems, such as iterative shrinkage-thresholding algorithms (ISTA) and its accelerated variants.


One further variant of the sparse coding problem is the double-sparse coding framework, which was proposed by Rubinstein et al. \cite{Rubinstein2010} and further studied by Nguyen et al. \cite{nguyen2019provably}. In this formulation, the dictionary $\mathbf{D}$ itself is assumed to have a sparse structure, expressed as:
$\mathbf{D} = \mathbf{\Psi} \mathbf{A}$
where $\b \Psi \in \mathbb{R}^{m \times m}$ is a ``base dictionary'' that is typically orthonormal (e.g., canonical or wavelet basis) and $\mathbf{A} \in \mathbb{R}^{m \times d}$ is a learned ``synthesis'' matrix whose columns are sparse.
This double-sparse structure can provide additional flexibility and efficiency in representing the data, since the sparse columns of $\mathbf{A}$ can be combined with the base dictionary $\b\Psi$ to form the overall dictionary $\mathbf{D}$. The sparsity of $\mathbf{A}$ also helps reduce the number of parameters that need to be learned, making the model more interpretable and potentially more robust to overfitting.


Convolutional sparse coding (CSC) is another powerful extension of the standard sparse coding problem, which was first introduced by Zeiler et al. \cite{5539957} in 2010. The key idea behind CSC is to take advantage of the convolutional structure of the dictionary $\mathbf{D}$ to enable the unsupervised construction of hierarchical image representations. By restricting the dictionary to a convolutional form, CSC models can effectively capture the spatially local and shift-invariant structures inherent in natural images.

Building on the success of CSC, Papyan et al. \cite{papyan2017convolutional} further proposed the Multilayer Convolutional Sparse Coding (ML-CSC) problem in 2017. In the ML-CSC framework, the dictionary $\mathbf{D}$ is modeled as a composition of several convolutional operators, enabling the encoding of progressively more complex features at multiple layers of abstraction. 
\begin{defn}[ML-CSC problem]\label{CSC}
	For a global observation signal $\b y$, a set of {\bf convolutional dictionaries} $\{\b D_j\}_{j=1}^{J}$, and a vector $\bm \lambda \in \RR^J_{+}$, we define the deep convolutional sparse coding problem $(\text{ML-CSC}_{\bm\lambda})$ as:
	\begin{equation*}
		\begin{array}{rlrrrr}
			(\text{ML-CSC}_{\bm\lambda}) \quad \text{find} \quad \{\b x_j\}_{j=1}^J  \quad
			\textrm{s.t.} \quad \b y&= \b D_1 \b x_1, \quad \|\b x_1\|_{0,\infty}^s \leq \lambda_1 \\
			\b x_1&=\b  D_2 \b x_2, \quad \|\b x_2\|_{0,\infty}^s \leq \lambda_2 \\
			&\vdots \\
		\b x_{J-1}&=\b D_{J}\b x_J, \quad \|\b x_J\|_{0,\infty}^s \leq \lambda_J,
		\end{array}
	\end{equation*}
	where $\lambda_j$ is the $j$-th entry of the vector $\bm \lambda$ and $\|\b x_j\|_{0,\infty}^s$ evaluate the maximum nonzero entries of stripe vectors of $\b x_j$. The norm $\|\cdot\|_{0,\infty}^s$ is specifically defined in \cite{papyan2017convolutional} to measure local sparsity.
\end{defn}
Papyan et al. \cite{papyan2017convolutional} also considered the ML-CSC model with noise, changing the constraints to $\|\b y - \b D_1 \b x_1\|_2 \leq \varepsilon_1$ and $\|\b x_j - \b D_{j+1} \b x_{j+1}\|_2 \leq \varepsilon_{j+1}$. Moreover, Papyan et al. \cite{papyan2017convolutional} made significant contributions to the theoretical understanding of sparse representations in the setting of Definition~\ref{CSC}. They established precise recovery conditions for the sparse codes, demonstrating that the iterative thresholding algorithm used for optimization is intimately connected to the core mechanisms of convolutional neural networks (CNN). This work provided a strong theoretical foundation for analyzing the sparse coding properties of deep learning architectures.

Sulam et al. \cite{sulam2018multilayer} adopted a principled projection-based approach and a sound pursuit algorithm for solving the ML-CSC problem. Their work not only advanced the practical implementation of ML-CSC, but also developed tighter theoretical bounds on the recovery performance of the corresponding algorithms. Addressing an even more general model, Aberdam et al. \cite{aberdam2019multi} revealed deep connections with fully connected neural networks. Their analysis established stronger uniqueness guarantees and tighter bounds for oracle estimators, further solidifying the theoretical understanding of the ML-CSC framework.
To simplify the complex optimization problem in deep pursuit, Sulam et al. \cite{sulam2019multi} introduced an assumption that the intermediate layers are not perturbed. This allowed them to consider sparse penalties in a single optimization problem, rather than a layer-by-layer approach. They proposed the ML-ISTA and ML-FISTA algorithms to efficiently solve this convex relaxation, and proved their convergence in terms of function value evaluation. Inspired by these developments, Rey et al. \cite{rey2020variations} explored additional convex alternatives and corresponding algorithms for the ML-CSC problem. Furthermore, Res-CSC and MSD-CSC \cite{zhang2021towards}, two variants of the ML-CSC model, were introduced to better understand the mechanisms of residual neural networks and mixed-scale dense neural networks, respectively.

Convolutional sparse coding (CSC) and its multilayer extensions are powerful frameworks that take advantage of the convolutional structure of the dictionary $\mathbf{D}$ to enable efficient and effective sparse representations. In these models, the dictionary $\mathbf{D}$ is restricted to a convolutional form, and the activation function involved in the related neural networks is typically the rectified linear unit (ReLU), defined as $\sigma(x) = \max\{x, 0\}$.

In this work, we go beyond the standard CSC and ML-CSC formulations and develop a more general sparse coding model. We establish the uniqueness and stability properties of the sparse representations in this generalized framework, providing a strong theoretical foundation for the sparse coding problem. Building on these insights, we then explore the deep connection between our general sparse coding problems and convolutional neural networks. We show that CNNs possess a remarkable capability in extracting sparse features, which can be viewed as a manifestation of sparse coding principles.
In addition, we further investigate the behavior of neural networks equipped with a wider range of activation functions and architectural designs. Inspired by these theoretical insights, we designed a feature sparsity training strategy and compared the resulting performance and features in numerical experiments. Specifically, the Deep Sparse Coding (DSC) problem under our consideration is given by:
\begin{defn}[Deep sparse coding problem]\label{assump:1}
	For a global noised signal $\b y$, a set of dictionaries $\{\b D_j\}_{j=1}^{J}$, a vector $\bm\lambda \in \RR^J_{+}$, and a tolerance vector $\bm\varepsilon \in \RR^J_{+}$, we call $\{\b x_j\}_{j=1}^J$ a set of sparse codings of $DSC_{0,\bm\lambda}^{\bm\varepsilon}$ if it satisfies
\begin{equation*}
	\begin{array}{rlrrrr}
		DSC_{0,\bm\lambda}^{\bm\varepsilon}(\b y, \{\b D_j\}_{j=1}^J) \quad \text{find} \quad \{\b x_j\}_{j=1}^J  \quad
		\textrm{s.t.} \quad \|\b y -\b D_1 \b x_1 \|_2 &\leq \varepsilon_1, \quad \|\b x_1\|_0 \leq \lambda_1, \\
		\|\b{x}_1 -\b D_2\b x_2\|_2 &\leq \varepsilon_2, \quad \|\b x_2\|_0 \leq \lambda_2,  \\
		&\vdots \\
		\|\b{x}_{J-1} -\b D_J \b{x}_J \|_2 &\leq \varepsilon_{J}, \quad \|\b{x}_J\|_0 \leq \lambda_J.
	\end{array}
\end{equation*}
\end{defn}
If we solve all feasible sets of sparse codings $\{\b x_j\}_{j=1}^J$ under constraint $\bm\varepsilon=\b 0$, then we denote problem $DSC_{0,\bm\lambda}^{\bm\varepsilon} := DSC_{0,\bm\lambda}^{\b 0}$. We call $\b y$ satisfying the problem $DSC_{0,\bm \lambda}^{\bm\varepsilon}$ if there exists at least one set of feasible sparse codes.

The Definition~\ref{assump:1} extends ML-CSC by introducing more flexible dictionaries $\b D_j$. The motivation behind these extensions stems from the practical observation that in image-related applications, different layers of neural networks extract feature representations at varying levels of abstraction. 
This hierarchical and multiscale nature of feature extraction is a hallmark of the success of deep learning in computer vision and other domains. By incorporating more flexible dictionaries and error tolerance, the Deep Sparse Coding (DSC) problem formulation aims to better capture this rich, multi-scale structure of feature representations. The increased flexibility allows the DSC model to adapt to the diverse characteristics of features at different network layers, rather than being constrained by the more rigid assumptions of the standard ML-CSC approach. Additionally, the error tolerance in the linear systems acknowledges the inherent approximations and imperfections that can arise when modeling the complex non-linear transformations performed by deep neural architectures.

{\bf Contributions:}  We summarize several key contributions to the field of deep sparse feature extraction by deep neural networks:
\begin{itemize}
\item Proposed DSC Models: We introduce a novel class of Deep Sparse Coding (DSC) models and establish a thorough theoretical analysis of their uniqueness and stability properties.
\item Convergence Rates via CNNs: By applying iterative algorithms to the DSC models, we derive convergence rates for convolutional neural networks in their ability to extract sparse features. This provides a strong theoretical foundation for the use of CNNs in sparse feature-learning tasks.
\item Extension to General Neural Networks: The main result of our work can be extended to general neural network settings, including those with diverse activation functions, as well as self-attention and transformer-based models. This broadens the applicability of our findings to a wide range of deep learning methods for the extraction of deep-sparse features.
\item Encouraging Sparse Feature Learning: Inspired by the strong connection between sparse coding and CNNs, we explore training strategies to encourage neural networks to learn sparser features. We conducted numerical experiments to evaluate the effectiveness of these approaches, providing valuable insight for the design of efficient deep learning algorithms.
\end{itemize}

\section{Uniqueness and stability of the deep sparse feature model}

In this section, we examine the conditions under which deep sparse coding problems admit a unique solution and whether this solution remains stable in the presence of noise. These results do not directly refer to neural networks themselves, but instead provide a theoretical study of the properties of the solution. When a solution is both unique and robust to noise, we can reasonably expect neural networks to learn this solution effectively.

\subsection{Uniqueness of deep sparse coding problems}\label{sec:unique}

In the classical sparse coding problem, one central challenge that arises when we translate problem $(P_0)$ to relaxed versions, $(P_0^\varepsilon)$ or $(P_{1,\lambda})$, is to assess the effectiveness of these solutions in recovering the sparse coding problem, $(P_0)$. One frequently used metric for this evaluation in the large literature of compressed sensing is the mutual coherence, which is also important in the following discussion.


\begin{defn}
	Denote the $k$-th column of $\b A \in \RR^{m\times d}$ by $\b a_k$. Then the mutual-coherence is given by
	\begin{align*}
		\mu(\b A) = \max_{i\neq j} \f{\left|\b a_i^\top \b a_j\right|}{\|\b a_i\|_2\|\b a_j\|_2}.
	\end{align*}
\end{defn}


%
The first central question is to identify the conditions under which the problem $(DSC_{0,\bm \lambda}^{\b 0})$ has a unique solution. We can essentially analyze uniqueness iteratively, starting from the first layer and proceeding to the deepest one, formulating the necessary conditions at each stage. The following result describes these necessary conditions.

\begin{thm}[Uniqueness via mutual coherence without noise]\label{unique}
	Given a set of dictionaries $\{\b D_j\}_{j=1}^{J}$ and consider a signal $\b y$, satisfying the problem $(DSC_{ 0,\bm\lambda}^{\b 0})$ and assume that $\{\b x_j\}_{j=1}^J$ is a solution to $(DSC_{0,\bm \lambda}^{\b 0})$. If for any $j$,
	\[ \|\b x_j\|_0 < \f{1}{2}\left(1+\f{1}{\mu(\b D_j)} \right) ,\]
	then $\{ \b x_j \}_{j=1}^{\infty}$ is the unique solution to the $(DCS_{0,\bm\lambda}^{\b 0})$ problem, provided that $\bm \lambda $ is chosen to satisfy
	\begin{align*}
		\lambda_j < \f{1}{2}\left(1+\f{1}{\mu(\b D_j)}\right).
	\end{align*}
\end{thm}

\begin{proof}
	 \cite[Theorem 2.5]{elad2010sparse} guaranties that under a mutual coherence condition, the sparsest solution of a system of linear equations is unique. More specifically, if $\b y = \b D_1 \b x_1$ has a solution satisfying $\|\b x_1\|_0 < \f{1}{2}\left( 1 + \f{1}{\mu(\b D_1)}\right)$, then $\b x_1$ is the unique and sparsest solution. Together with condition $\lambda_1 < \f{1}{2}\left( 1 + \f{1}{\mu(\b D_1)}\right)$, we conclude that $\b x_1$ is unique. Similarly and iteratively, the conditions guarantee that $\b x_j $ is unique for $j=2,\dots,J$, which leads to the result.
\end{proof}

Notice that \defref{assump:1} and \thmref{unique} are similar to \defref{CSC} in \cite{papyan2017convolutional}. However, the conditions made here are weaker in the sense that the dictionaries in \cite{papyan2017convolutional} are required to be convolutional.
Two additional approaches to further relaxing the uniqueness condition were proposed in \cite{aberdam2019multi,sulam2018multilayer}. 
Noticing that for $(DSC_{ 0,\bm\lambda}^{\b 0})$, we have $\b y = \b D_1 \cdots \b D_j \b x_j$ and $\b x_j = \b D_{j+1} \cdots \b D_{j+m} \b x_{j+m}$. We denote $\b D_{[j]}:= \b D_1 \cdots \b D_j$ and $\b D_{[j,j+m]}:= \b D_{j} \cdots \b D_{j+m}$ for simplicity. For the uniqueness of a sparse coding $\b x_{j_0}$ from layer $j_0$, we can consider the solutions of $\b y = \b D_{[j_0]} \b x_{j_0}$ or $\b x_{j} = \b D_{[j,j_0]} \b x_{j_0}$ for any $j<j_0$ if $\b x_j$ is unique. Or more specifically, if the uniqueness of $\b x_j$ is already satisfied, together with the same idea used in \thmref{unique}, we only need to require $\lambda_{j_0}$ to be bounded by
\begin{align}\label{eq:j0}
	\lambda_{j_0} < \f{1}{2} \max \left \{ 1+\f{1}{\mu(D_{j_0})}, 1+\f{1}{\mu(D_{[j_0]})}, 1+\f{1}{\mu( D_{[j, j_0]} )}: j<j_0   \right \},
\end{align}
then we can ensure that $\b x_{j_0}$ is unique for $(DSC_{ 0,\bm\lambda}^{\b 0})$. This approach is a top-bottom layer-by-layer approach.

Let us think of extracting sparse codings simultaneously and conversely. Assume that we know the support of each $\b x_j$. We denote the support of $\b x_j$ as $\S_j$ and collect its nonzero elements in $\b x_{j, \S_j}$. Let $\b A_{\S_i, \S_j}$ be the submatrix of matrix $\b A$ with the elements selected from the row index set $\S_i$ and the column index set $\S_j$. Collecting non-zero elements of $\b x_j$ rowwise, then solving $(DSC_{ 0,\bm\lambda}^{\b 0})$ is equivalent to solving the following linear system
\begin{align}\label{eq:cosparsity}
	\begin{pmatrix}
		\b x_{1, \S_1^c}\\
		\b x_{2, \S_2^c} \\
		\vdots \\
		\b x_{J-1, \S_{J-1}^c}
	\end{pmatrix} =
\begin{pmatrix}
	\left(\b D_{[2, J]}\right)_{\S_1^c, \S_{J}} \\
	 \left(\b D_{[3, J]}\right)_{\S_2^c, \S_{J}} \\
	\vdots \\
	 \left(\b D_{[J,J]}\right)_{\S_{J-1}^c, \S_{J}}
\end{pmatrix}
\b x_{J, \S_J} := \b D^{[\S_J]} \b x_{J, \S_J}.
\end{align}
Since the left-hand side of \eqref{eq:cosparsity} is actually a zero vector, to get $\b x_{J, \S_J}$, we only need to search for the null space of $\b D^{[\S_J]}$, which is of dimension $\lambda_J - \rank(\b D^{[\S_J]})$ and smaller than $\lambda_J$. After getting the deepest sparse coding and pushing back through the equation $\b x_{j-1} = \b D_{j}\b x_j$, we completely solve $(DSC_{ 0,\bm\lambda}^{\b 0})$. Although this idea helps to relax the conditions and the complexity, usually one does not know the supports of the underlying sparse features, and hence the second method does not match well with general applications of interest. The work \cite{aberdam2019multi} provides a more detailed discussion of this method.

After examining the condition for uniqueness, it is natural to ask whether this problem can be addressed using certain algorithms.
Following a standard relaxation, we can also establish the relationship between $(DSC_{0,\bm \lambda}^{\b 0})$ and the following deep convex optimization problem $(DSC_{1})$, which could be solved using convex optimization algorithms.
\begin{defn}[$DSC_{1}$]\label{DSCl1}
	For global observation signal $\b y$, a set of dictionaries $\{\b D_j\}_{j=1}^{J}$, define the deep coding problem $(DSC_{1})$ as:
	\begin{equation*}
		\begin{array}{lccccc}
			DSC_{1}\left(\b y, \{\b D_j\}_{j=1}^{J}\right) \quad \text{find} \quad \{x_j\}_{j=1}^J  \quad
			\textrm{s.t.} \quad & \b x_1 := {\operatorname{argmin}} \left \{\|\b x\|_1: \b y= \b D_1 \b x\right\} , \\
			&\b x_2 := {\operatorname{argmin}} \left \{\|\b x\|_1:\b x_1= \b D_2 \b x \right \}, \\
			&\vdots \\
			&\b x_{J} := {\operatorname{argmin}} \left \{\|\b x\|_1: \b x_{J-1}= \b D_J \b x \right \}.
		\end{array}
	\end{equation*}
\end{defn}

The next theorem characterizes the case in which the two problems have the same unique solution.

\begin{thm}[Coincidence between $\ell_0$ and $\ell_1$]\label{thm:dsc1}
	Let $\{\b D_j\}_{j=1}^J $ be a set of dictionaries, with each $\b D_j$ being of full rank and the number of rows less than that of columns. Assume that the conditions in \thmref{unique} are satisfied and that $\{\b x_j\}_{j=1}^J$ is a solution of $DSC_{0,\bm\lambda}^{\b 0}$. Then $\{\b x_j\}_{j=1}^J$ is the unique solution of $DSC_{0,\bm\lambda}^{\b 0}$ and $DSC_{1}$.
\end{thm}

\begin{proof}
	Theorem 2.5 and Theorem 4.5 in \cite{elad2010sparse} jointly imply that if there exists a solution $\b x_1$ of $\b y = \b D_1 \b x_1 $ such that $\|\b x_1\|_0 \leq \f{1}{2}(1 + \f{1}{\mu(\b D_1)})$, then $\b x_1$ is necessarily the sparsest possible and $\b x_1 = \operatorname{argmin} \{ \|\b x\|_1: \b y = \b D_1 \b x \}$. Hence, iteratively we are able to show the statement.
\end{proof}

\thmref{thm:dsc1} indicates that we can use optimization algorithms to solve $(DSC_{0,\bm \lambda}^{\b 0})$ by solving $DSC_{1}$, and the convergence analysis can be performed iteratively. In the next section, we shall construct CNNs that can solve $(DSC_{0,\bm \lambda}^{\b 0})$ with exponential decay.

\subsection{Stability of deep sparse  problem}
In practical applications, noise is typically present, arising either from the sampling process or computational procedures. How are the solutions affected when such noise appears? Before examining the robustness of the deep sparse coding problem, and to streamline the exposition in this work, we first introduce the following uniqueness condition.

\begin{assump}[Uniqueness condition]\label{assump:2}
	The parameters $\bm\lambda$ and $\b D_j$ of \defref{assump:1} satisfy the following sparsity condition
	\begin{align*}
		\lambda_j < \f{1}{2}\left(1+\f{1}{\mu(\b D_j)}\right),  \forall j \in [J].
	\end{align*}
\end{assump}
Assumption~\ref{assump:2} ensures that, in the absence of noise, the solution is unique. In the presence of noise, we can instead analyze how far the obtained solution deviates from this unique one. The underlying idea is similar to the argument used in the proof of Theorem~\ref{unique}: once stability is established for the first layer, a similar reasoning can be iteratively applied to obtain the corresponding results for all subsequent layers.

\begin{thm}[Stability of $(DSC_{0,\bm \lambda}^{\bm\varepsilon})$ problem]\label{thm:dsc_st}
	
	Given a set of dictionaries $\{\b D_j\}_{j=1}^J$ with each dictionary $\b D_j$ being column-normalized with respect to $\ell_2$ norm and the sparsity condition $\bm\lambda $ satisfies \assref{assump:2}. Let $\b y$ be a global observation that satisfies \defref{assump:1}.
	If the collections $\{\b x_j\}_{j=1}^J$ and $\{\tilde{\b x}_j\}_{j=1}^J$ are solutions to $DSC_{0,\bm\lambda}^{\b 0}(\b y, \{\b D_j\}_{j=1}^J)$ and $DSC_{0,\bm\lambda}^{\bm\varepsilon}(\b y, \{\b D_j\}_{j=1}^J)$, respectively,
	then they must obey
	\begin{align*}
	\|\tilde{\b x}_j - \b x_j \|_2 \leq \f{\sum_{i=1}^{j}\varepsilon_i\prod_{m=1}^{i-1}\sqrt{1-(2\lambda_m-1)\mu(\b D_m)}}{\prod_{i=1}^{j}\sqrt{1-(2\lambda_i-1)\mu(\b D_i)}}.
	\end{align*}
\end{thm}
\begin{proof}
	The proof is partially inspired by Section 5.2.3 in \cite{elad2010sparse}.
	Let $\b Q = \b 1 \b 1^\top$. Then it is easy to see that for any $\b x$ with a sparsity $\|\b x\|_0 \leq s$,
	\begin{align}\label{key1}
		\begin{aligned}
			\|\b D\b x\|_2^2 = \b x^\top \b D^\top \b D \b x &\geq \b x^\top \b x + \mu(\b D)|\b x|^\top(\b I - \b Q)|\b x| \\
			&\geq \left(1+ \mu(\b D)\right)\|\b x\|_2^2 - \mu(\b D) \|\b x\|_1^2 \\
			&\geq \left(1+ \mu(\b D)\right)\|\b x\|_2^2 - \mu(\b D)s \|\b x\|_2^2 \\
			&= \left(1-(s-1) \mu(\b D)\right)\|\b x\|_2^2 ,
		\end{aligned}
	\end{align}
	where in the third step we use the inequality $\|\b c\|_1 \leq \sqrt{s}\|\b c\|_2$ for any $\b c \in \RR^s$ and in the first step we use the following fact
	\begin{align*}
		\b x^\top \b D^\top \b D \b x = \|\b x \|_2^2 + \sum_{i \neq j} \b d_i^\top \b  d_j x_i x_j \geq \|\b x \|_2^2 - \sum_{i \neq j} \mu(\b D) |x_i x_j| = \|\b x \|_2^2 - \mu(\b D)|\b x|^\top (\b Q-\b I) |\b x|.
	\end{align*}
	Since $\b y=\b D_1 \b x_1$ and $\|\b y - \b D_1\tilde{\b x}_1\|_2 \leq   \varepsilon_1$, we have $\|\b D_1\b x_1 - \b D_1\tilde{\b x}_1\|_2 \leq  \varepsilon_1$. Combining this with \eqref{key1} and $\|\b x_1 - \tilde{\b x}_1\|_0 \leq 2\lambda_1$, we have
	\begin{align}\label{key2}
		\|\b x_1 - \tilde{\b x}_1\|_2^2 \leq \f{  \varepsilon_1^2}{1-(2\lambda_1-1)\mu(\b D_1)}.
	\end{align}
	Again, combining $\|\tilde{\b x}_1 - \b D_2\tilde{\b x}_2\|_2 \leq  \varepsilon_2$, $\b x_1 = \b D_2\b x_2$, and \eqref{key2}, we obtain the following inequality
	\begin{align}\label{key3}
		\|\b D_2\b x_2 - \b D_2\tilde{\b x}_2\|_2 \leq  \varepsilon_2 + \f{ \varepsilon_1}{\sqrt{1-(2\lambda_1-1)\mu(\b D_1)}}.
	\end{align}
	Denote
	\begin{align*}
		\delta_1 := \f{ \varepsilon_1}{\sqrt{1-(2\lambda_1-1)\mu(\b D_1)}}.
	\end{align*}
	Then by \eqref{key1} and \eqref{key3}, we get
	\begin{align*}
		\|\b x_2 - \tilde{\b x}_2\|_2^2 \leq \f{( \varepsilon_2 + \delta_1)^2}{1-(2\lambda_2-1)\mu(\b D_2)},
	\end{align*}
	Repeating the above step for $j$ from $3$ to $J$, we conclude that
	\begin{align*}
		\|\b x_j - \tilde{\b x}_j\|_2^2 \leq \f{( \varepsilon_j + \delta_{j-1})^2}{1-(2\lambda_{j}-1)\mu(\b D_j)},
	\end{align*}
	with $\delta_{j}$ defined iteratively as
	\begin{align*}
		\delta_j = \f{ \varepsilon_j + \delta_{j-1}}{\sqrt{1-(2\lambda_j-1)\mu(\b D_j)}}.
	\end{align*}
\end{proof}

The following corollary offers a clearer formulation that, for instance, fits naturally with the image denoising setting. It shows that in the presence of noise, the obtained solutions remain close to the unique exact solution, with their distance bounded by the noise level multiplied by certain constants.

\begin{cor}\label{cor:noiseless}
	Assume that the conditions of \thmref{thm:dsc_st} hold and we choose $\bm\varepsilon := (\varepsilon_1,0,\dots,0)^\top \in \RR^{J}$.
	If the collections $\{\b x_j\}_{j=1}^J$ and $\{\tilde{\b x}_j\}_{j=1}^J$ satisfy the model $DSC_{0,\bm\lambda}^{\b 0}(\b y, \{\b D_j\}_{j=1}^J)$ and $DSC_{0,\bm \lambda}^{\bm\varepsilon}(\b y, \{\b D_j\}_{j=1}^J)$, respectively,
	then they must obey
	\begin{align*}
		\|\tilde{\b x}_j - \b x_j \|_2^2 \leq \|\bm\varepsilon\|_2^2 \left( \prod_{i=1}^{j}( 1- (2\lambda_j-1)\mu(\b D_j))\right)^{-1}.
	\end{align*}
\end{cor}



As pointed out in \cite{papyan2017convolutional}, results for convolutional dictionaries can be applied to general dictionaries by viewing the convolution kernel as the same size as the column of dictionaries (in this way their results do not work for convolutional neural networks). In \cite{papyan2017convolutional, aberdam2019multi}, they study the same setting as in \corref{cor:noiseless}, and their resulting upper bound is
\begin{align*}
	4\|\bm\varepsilon\|_2^2 \left( \prod_{i=1}^{j}( 1- (2\lambda_j-1)\mu(\b D_j))\right)^{-1} .
\end{align*}
Compared with this result, our bound is tighter for general dictionaries.
In \cite{sulam2018multilayer}, a similar setting as \thmref{thm:dsc_st} 
for convolutional dictionaries are discussed but our result generalizes the dictionaries to be arbitrarily chosen.

The stability analysis outlined above is still carried out in a layer-by-layer manner. By applying the idea introduced for \eqref{eq:j0}, the conclusion of \thmref{thm:dsc_st} can be extended to
\begin{align*}
	\|\tilde{\b x}_{j_0} -\b x_{j_0} \|_2^2 \leq \f{( \varepsilon_{[j,j_0]} + \delta_{j})^2}{1-(2\lambda_{j_0}-1)\mu(\b D_{[j,j_0]})},\quad j < j_0,
\end{align*}
when we have the noise information between layer $j$ and $j_0$, i.e. $\|\b x_j - \tilde{\b x}_j\|_2 \leq \delta_j$ and $\|\tilde{\b x}_j - \b D_{[j,j_0]} \tilde{\b x}_{j_0}\|_2 \leq \varepsilon_{[j,j_0]}$. In this way, the error bound does not cumulate from $j$-th layer to $j_0$-th layer. If one would like to consider all sparse codings as \eqref{eq:cosparsity} in a single equation, i.e., we directly use the deepest information $\|\b D^{[\S_J]} \b x_{J,\S_J} - \b D^{[\S_J]} \tilde{\b x}_{J,\S_J}\|_2 \leq \varepsilon$, 
we immediately obtain the following bound
\begin{align*}
	\| \b x_{J,\S_J} -  \tilde{\b x}_{J,\S_J}\|_2 \leq \varepsilon \left(1-(2\lambda_J-1)\mu(\b D^{[\S_J]}) \right)^{-1/2}.
\end{align*}
Furthermore, denote
\begin{align*}
	\tilde{\b x}_{j, \S_{j}}:=\left(\b D_{[j,J]}\right)_{\S_j, \S_J} \tilde{\b x}_{J, \S_{J}}.
\end{align*}
It is straightforward to see that layer-$j$ error is given by
\begin{align*}
	\| \b x_{j,\S_j} -  \tilde{\b x}_{j,\S_j}\|_2 \leq \varepsilon \left\|\left(\b D_{[j,J]}\right)_{\S_j, \S_J}\right\|_2 \left(1-(2\lambda_J-1)\mu(\b D^{[\S_J]}) \right)^{-1/2},
\end{align*}
which also yields a noncumulative bound. The above discussion suggests that, in applications, the conditions required could be more optimistic than those required in \thmref{thm:dsc_st}.

\section{Deep sparse feature extraction via CNNs}

In this section, we investigate whether CNNs can serve as solvers for deep sparse coding problems. As outlined earlier, there are two possible strategies for learning sparse representations: a layer-by-layer approach and a deepest-sparsity-first approach. Although these two approaches are based on different conditions, they yield comparable convergence rates. These analysis focus on the ReLU activation function. Given that many ReLU-type variants are also widely employed in various deep learning tasks, we extend our results to such networks and show that the sparse coding properties still hold. This suggests that the capability for sparse feature learning may be broadly present in deep learning models. 
To concentrate on presenting and interpreting our main theoretical results, several proofs in this section are deferred to the Appendix.

\subsection{Convolutional neural networks}

Let $s,t,d$ be integers. Throughout, we write $\D(d,t):=\lceil d/t\rceil$ for the smallest integer greater than or equal to $d/t$, and use $A_{i,j}$ to denote the $(i,j)$-th entry of a matrix $\b A\in \RR^{m\times n}$, where we set $A_{i,j}=0$ whenever $i\notin[1,m]$ or $j\notin[1,n]$. For convenience, we use $\b 1_{m\times n} := (1)_{i\in [n], j\in [m]}$ to denote the $m\times n$ matrix whose entries are all equal to $1$. The same convention applies to vectors. 

The 1D convolution of size $s$, stride $t$ is computed for a kernel $\b w\in \RR^s$ and a signal $ \b x\in \RR^d$ to be $\b w*_t  \b x\in \RR^{\D(d,t)}$ by
  \begin{equation*}
  (\b w *_t \b x )_{i}=\sum_{k=1}^{s} w_{k} x_{(i-1)t+k}, \ i=1,\ldots, \D(d,t).
  \end{equation*}
A multichannel CNN block consists of more than one convolutional layer and multiple channels in each layer, which is defined as follows.
\begin{defn}\label{1dcnn}
Given the input size $d \in \mathbb{N}$, a multi-channel CNN block in 1D of depth $J$ with respect to filter size $\{s_j\geq 2\}_{j=1}^J$, stride $\{t_j\geq 1\}_{j=1}^J$ and channel $\{n_j\geq 1\}_{j=1}^J$,
 is a neural network $\left\{h_{j}: \mathbb{R}^{d} \rightarrow \mathbb{R}^{d_j\times n_j}\right\}_{j=1}^{J}$ defined iteratively for which  $h_{0}(\b x)=\b x \in \mathbb{R}^{d}$, $d_0=d$, $n_0=1$, and the $\ell$-th channel of $h_{j}$ is given by
$$
\begin{aligned}
    h_{j}(\b x)_{\ell}&=\rho\left(\sum_{i=1}^{n_{j-1}}\b  w^{(j)}_{\ell,i}*_{t_j} h_{j-1}(\b x)_i+b^{(j)}_\ell \b{1}_{d_j} \right),\quad \ell = 1,\dots,n_j,\\
    h_{J}(\b x)_{\ell}&=\sum_{i=1}^{n_{J-1}}\b  w^{(J)}_{\ell,i}*_{t_J} h_{J-1}(\b x)_i+b^{(J)}_\ell \b{1}_{d_J} \quad \quad \ell = 1,\dots,n_J,
\end{aligned}
$$
with $j=1, \ldots, J-1$, where $b^{(j)}_\ell \in \mathbb{R}$ are biases, $\b w^{(j)}_{\ell,i}\in \RR^{s_j}$ are filters, $d_j=\D(d_{j-1},t_j)$, and $\rho:\RR\rightarrow \RR$ is an activation function which is applied to each entry of a given vector or matrix. We highlight the role of the activation function $\rho$ in multi-channel CNNs by referring to them as $\rho$-activated CNNs.
\end{defn}

In the following, to further deepen the CNNs defined in \defref{1dcnn}, whenever the output features of each channel in certain layers are one-dimensional, we treat them as vectors with a single channel. Hence, the whole CNN architecture can be given by compositions of CNN blocks. This allows CNNs to reach an arbitrary depth when for example, we choose $t_j=s_j$ for all $j$, thus enhancing their learning capacity. For simplicity, we also refer to the output $h_j(\bm x)$ of the $j$-th hidden layer as the CNN output. We will only refer to outputs of some intermediate layers of a CNN without specifying their layer indices. To investigate the learning capability of CNNs, we represent the influence of network size on error bounds using two quantities: the number of layers and the total number of weights across all filters and biases.

The vast majority of CNNs in image processing are induced by 2D convolutions. As an analogy of the $1$D case, 2D convolutions and CNNs can be well defined. Precisely, the 2-D convolution of a filter $\b W\in \RR^{s\times s}$ and an image $\b X\in \RR^{d\times d}$ with a stride $t$ produces a convoluted matrix $\b W \circledast _t \b X\in \RR^{\D(d,t)\times \D(d, t) }$ for which the $(i,j)$-entry is given by
$$
(\b W \circledast_t \b X)_{(i,j)}:=\sum_{\ell_1,\ell_2=1}^s W_{\ell_1,\ell_2} X_{(i-1)t+\ell_1, (j-1)t+\ell_2}, \quad i,j=1,\ldots, \D(d,t).
$$
The 2D CNNs can be built by applying multi-channel 2D convolution operations in accordance with \defref{1dcnn}. Under appropriate assumptions, it has been shown that the analysis of 2D CNNs can be reduced to the study of 1D CNNs \cite{FLZ}. Therefore, without loss of generality, we restrict our attention to 1D CNNs in this paper.

\subsection{Fast decay rate in solving deep sparse coding via ReLU-activated CNNs}

The sparse feature extraction capability of CNNs was studied in \cite{papyan2017convolutional}, where a CNN layer is linked to a thresholding algorithm. Meanwhile, algorithms for solving $(P_0^\varepsilon)$, as well as the relationship between its solutions and those of $(P_0)$, have been extensively investigated in the literature. If we can establish an analogous connection between CNNs and a solution algorithm for $(P_0^\varepsilon)$, then the analysis of CNNs for deep sparse coding problems can be extended in a way similar to our study of uniqueness and stability. 

From the viewpoint of approximation theory, this is feasible. The universal approximation property of CNNs has been established in \cite{zhou2020universality}, along with their strong capability in feature extraction and in approximating functions with certain compositional structures \cite{mao2021theory}. Inspired by these works, the authors of \cite{FLZ} generalized these CNN architectures to multichannel CNNs, demonstrating that any linear transformation can be implemented by CNNs. These results indicate that many algorithms can be equivalently realized using CNN architectures.
Based on this observation, we can rigorously characterize the learning capacity of CNNs in the context of deep sparse coding problems. The treatment of the case $J=1$ is postponed to Appendix~\ref{app:1}; in what follows, we will focus primarily on presenting results for deep sparse coding problems.

The method of the deepest-sparsity-first approach prioritizes the approximation of sparse codes in the deepest layer of the $(DSC_{0,\bm\lambda}^{\bm \varepsilon})$ model. By focusing first on the most abstract latent representations, it establishes a foundational structure that informs the approximation of shallower sparse coding problems. To guarantee the uniqueness property, we need to impose additional conditions on $\b x_J$ and $\lambda_J$.

\begin{thm}[Deepest-sparsity-first approach]\label{l1DSC}
	Let $J \in \NN$ and $d_0, d_1, \dots, d_{J+1} \in \NN$ with $d_{J+1} :=d_0$. Consider a sequence of column-normalized dictionaries $\{\b D_j \in \RR^{d_{j-1} \times d_j} \}_{j=1}^J$, satisfying Assumption \ref{assump:2} with sparsity parameter $\bm \lambda$. 
    For a global observation $\b y$ generated by the $(DSC_{0,\bm\lambda}^{\bm \varepsilon})$ model, let $\{\b x_j\}_{j=1}^J$ be a solution to the corresponding problem $DSC_{0,\bm \lambda}^{\bm  \varepsilon}(\b y, \{\b D_j\})$. Under the constraints:
	\begin{enumerate}
		\item[(1)] $\|\b{x}_J \|_\infty \leq B_J$ for some $B_J > 0$, and
		\item[(2)] $\lambda_J < \frac{1}{2}\left(1 + \frac{1}{\mu(\b D_{[J]})}\right)$,
	\end{enumerate}
	there exists a ReLU-activated CNN architecture with 
		 kernel size $s $,
		 depth $O\left(K \log_s \prod_{j=1}^{J+1}(d_{j-1}+d_{j})\right)$,
		 and number of weights $O\left( K \sum_{j=1}^{J+1}(d_{j-1}+d_j)^2 \right)$,
	such that its output sequence $\{\tilde{\b x}_j\}_{j=1}^J$ satisfies the approximation bound:
	\begin{align*}
		\|\tilde{\b x}_j - \b x_j \|_2 \leq  C_{1} e^{-c K} + C_{2} \|\bm \varepsilon\|_\infty,
		\quad \forall j \in [J],
	\end{align*}
	where the constants $c, C_{1}, C_{2} > 0$ depend solely on $\{\b D_j \}_{j=1}^J$, $\bm \lambda$, and $B_J$.
\end{thm}

In contrast, the sequential layer-by-layer strategy iteratively refines approximations starting from the initial layer and progressing toward deeper layers. Using incremental updates, it propagates constraints through the hierarchy, ensuring consistency across all levels.

\begin{thm}[Layer-by-layer approximation rate]\label{thm:layer-by-layer}
	Let $J \in \NN$ and $d_0, d_1, \dots, d_{J} \in \NN$. Consider a sequence of column-normalized dictionaries $\{\b D_j \in \RR^{d_{j-1} \times d_j} \}_{j=1}^J$, satisfying Assumption \ref{assump:2} with sparsity parameter $\bm \lambda$. 
    For a global observation $\b y$ generated by the $(DSC_{0,\bm\lambda}^{\bm \varepsilon})$ model, let $\{\b x_j\}_{j=1}^J$ be a solution to the corresponding problem $DSC_{0,\bm \lambda}^{\bm  \varepsilon}(\b y, \{\b D_j\})$. Under the additional constraints:
	\begin{enumerate}
		\item[(1)] $\|\b{x}_j \|_\infty \leq B$, $ j \in [J]$, for some $B>0$ and
		\item[(2)] $\ln d_{j} \leq C$ for any $1\leq j \leq J$,
	\end{enumerate}
	there exists a ReLU-activated CNN architecture with 
		 kernel size $s $,
		 depth $O\left(K \log_s \prod_{j=1}^{J}(d_{j-1}+d_{j})\right)$,
		 and number of weights $O\left( K \sum_{j=1}^{J}(d_{j-1}+d_j)^2 \right)$,
	such that its output sequence $\{\tilde{\b x}_j\}_{j=1}^J$ satisfies the approximation bound:
	\begin{align*}
		\|\tilde{\b x}_j - \b x_j \|_2 \leq  C_1 e^{-\frac{1}{2} K} + C_2 \|\bm \varepsilon\|_\infty \sum_{i=1}^j e^{\left (2 a^{J-i}-\frac{1}{2}  \right) K},
		\quad \forall j \in [J],
	\end{align*}
	for large enough $K\in\NN$, where the constants $a>1, C_{1}, C_{2} > 0$ depend solely on $\{\b D_j \}_{j=1}^J$, $\bm \lambda$, $C$, and $B$.
\end{thm}

Both theorems establish CNN-based approximation guarantees for deep sparse coding problems, though they diverge in their strategies, constraints, and error propagation dynamics. The deepest-sparsity-first approach (Theorem \ref{l1DSC}) imposes stricter conditions on the final layer, while the layer-by-layer method (Theorem \ref{thm:layer-by-layer}) exhibits greater sensitivity to noise-induced errors. This discrepancy arises because early-layer approximation errors in the layer-by-layer strategy propagate cumulatively through the network, amplifying perturbations in shallow layers. Notably, both frameworks benefit from increased network depth ($K$), yielding exponential error decay in the absence of noise $\bm{\varepsilon} = \mathbf{0}$. In this case, the CNN outputs approximate the unique solutions in \thmref{unique}. 
The constants appearing in Theorem \ref{l1DSC} and Theorem \ref{thm:layer-by-layer} are primarily determined by the norms of the dictionaries, the dimensions of the solutions, and the parameter $\bm\lambda$. When these quantities are small, the corresponding constants are also expected to be small. Because the setup of the deep sparse coding problem is fixed, these constants can be regarded as fixed as well. Their explicit forms are provided in the proof.

For the noiseless case, we shall henceforth focus on the corollaries of Theorem \ref{l1DSC} to elucidate the representational capacity of CNNs in approximating hierarchical sparse features. This analysis directly extends to Theorem \ref{thm:layer-by-layer} with minimal modifications, as the error decay mechanisms in the noiseless regime are structurally analogous.
 
\begin{cor}[Deep sparse coding without noise]\label{cor:no}
	Under the conditions of Theorem \ref{l1DSC} with $\bm \varepsilon = \bm 0$, there exists a ReLU-activated CNN architecture with kernel size $s $, depth $O\left(K \log_s \prod_{j=1}^{J+1}(d_{j-1}+d_{j})\right)$, and number of weights $O\left( K \prod_{j=1}^{J+1}(d_{j-1}+d_j)^2 \right)$ such that its output sequence $\{\tilde{\b x}_j\}_{j=1}^J$ satisfies:
	\begin{align*}
		\|\tilde{\b x}_j - \b x_j \|_2 = O \left(e^{-c K} \right),
		\quad \forall j \in [J],
	\end{align*}
	where $c > 0$ depends solely on $\{\b D_j \}_{j=1}^J$ and $\bm \lambda$. 
	
	Equivalently, for any $\varepsilon > 0$, there exists a CNN architecture with
depth $O\left(\ln \frac{1}{\varepsilon}\right)$ and number of weights $O\left(  \ln \frac{1}{\varepsilon} \right)$
that achieves $\|\tilde{\b x}_j - \b x_j\|_2 \leq \varepsilon$ for all $j \in [J]$.
\end{cor}

\corref{cor:no} follows in a rather direct way from \thmref{l1DSC}, yet it is quite interesting in the context of curse-of-dimensionality issues arising in approximation theory. Modern approximation theory has become indispensable for understanding the expressive capacity of deep neural networks. However, classical approximation results for function classes such as Sobolev spaces $ W^{r,p}(\Omega)$ \cite{yarotsky2017error, petersen2018optimal, han2022feature, guhring2020error} and continuous functions \cite{Yarotsky2018optimal, shen2019deep, shen2022optimal} with $ \Omega \subset \mathbb{R}^d $ underscore a critical limitation: the \textit{curse of dimensionality}. Specifically, to approximate a function $ f \in W^{r,\infty}(\Omega) $ within an error $ \varepsilon $, a fully connected neural network should have at least $ O(\varepsilon^{-d/r}) $ nonzero weights \cite{yarotsky2017error}. This exponential scaling with respect to $ d $ renders such bounds impractical for high-dimensional problems.

To address this challenge, previous work has focused on two strategies:
\begin{itemize}
    \item Restrictions to function spaces: By targeting structured subclasses (e.g. \textit{Korobov spaces}), the complexity bounds are improved \cite{montanelli2019new,liu2024approximation,li2024signrelu, montanelli2021deep, Klusowski2018Approximation,petersen2018optimal}.
    \item Manifold assumptions: If inputs lie on a low-dimensional manifold $ \mathcal{M} \subset \mathbb{R}^d $ with intrinsic dimension $ m \ll d $, then approximation rates depend mainly on $ m $ \cite{liu2021besov,chui2018deep,shaham2018provable,schmidt2019deep,nakada2020adaptive,chen2019efficient,zhou2020theory}.
\end{itemize}
 
Corollary~\ref{cor:no} provides a complementary approach based on the sparsity-aware approximation. By exploiting the inherent sparsity of real-world data, we demonstrate that the curse of dimensionality can be mitigated.
This result highlights the practical efficacy of sparsity. 

We can also generalize the $\ell_2$ error measurement to $\ell_2$-$\ell_1$ loss.


\begin{cor}[Exponential convergence w.r.t. $\ell_2$-$\ell_1$ loss]\label{cor:l2l1}
	Given $\gamma > 0$, define the $\ell_2$-$\ell_1$ loss function as:
	\begin{align}\label{eq:l2l1_loss}
		L_j(\b x) := \| \b x_{j-1} - \b D_j \b x \|_2^2 + \gamma \| \b x \|_1
	\end{align}
	where $\b x_0 := \b y$ denotes the initial observation. Under the conditions of Corollary \ref{cor:no}, there exists a ReLU-activated CNN architecture with kernel size $s $, depth $O\left(K \log_s \prod_{j=1}^{J+1}(d_{j-1}+d_{j})\right)$, and number of weights $O\left( K \sum_{j=1}^{J+1}(d_{j-1}+d_j)^2 \right)$ such that its output sequence $\{\tilde{\b x}_j\}_{j=1}^J$ satisfies the convergence bound:
	\begin{align}\label{eq:loss_bound}
		0 \leq L_j(\tilde{\b x}_j) - L_j(\b x_j) \leq C e^{-cK},
		\quad \forall j \in [J],
	\end{align}
	where constants $c, C > 0$ depend solely on $\{\b D_j \}_{j=1}^J$, $\bm \lambda$, $B_J$, and $\gamma$.
\end{cor}

\subsection{General network configurations for solving deep sparse coding problems}

Although the aforementioned results primarily focus on CNNs with ReLU activation functions, a broad application of alternative activation functions---such as Leaky ReLU \cite{maas2013rectifier}, PReLU \cite{he2015delving}, and Swish \cite{ramachandran2017searching,elfwing2018sigmoid}---exhibit distinct advantages in specific tasks, including adaptive parameterization, or enhanced nonlinearity \cite{dubey2022activation}. Given the widespread adoption of these variants across diverse applications, extending our theoretical framework to accommodate general activation functions is practically valuable. In what follows, we formalize this extension by analyzing the approximation properties of CNNs equipped with ReLU-type activations, thus broadening the scope of our prior results.

\begin{defn}[ReLU-type activation functions]\label{thm:gen_relu}
	An activation function $\rho: \RR \to \RR$ is termed a $(L,\beta)$-ReLU-type activation function if it meets the following criteria:
	\begin{enumerate}
		\item[(1)] Continuity: $\rho \in C(\RR)$;
		\item[(2)] Piecewise linearity: $\rho(x) = x$ for $x \geq 0$;
		\item[(3)] Bounded negative part: $|\rho(x)| \leq \beta$ for $x < 0$;
		\item[(4)] Lipschitz property: $|\rho(x) - \rho(y)| \leq L|x - y|$ for $x,y \leq 0$.
	\end{enumerate}
\end{defn}

The following result provides a similar characterization as Theorem \ref{l1DSC} for ReLU-type activation functions:

\begin{thm}\label{thm:relutype}
	Let $J, t \in \NN$, $L, \beta \geq 0$, and $d_0, d_1, \dots, d_{J+1} \in \NN$ with $d_0 = d_{J+1}$. Let $\rho$ be a $(L,\beta)$-ReLU-type activation function satisfying the conditions of Definition \ref{thm:gen_relu}. Consider a sequence of column-normalized dictionaries $\{\b D_j \in \RR^{d_{j-1} \times d_j} \}_{j=1}^J$ satisfying Assumption \ref{assump:2} with sparsity parameter $\bm \lambda$. For a global observation $\b y$ generated by the $(DSC_{0,\bm\lambda}^{\bm \varepsilon})$ model, let $\{\b x_j\}_{j=1}^J$ denote a solution to the corresponding $DSC_{0,\bm \lambda}^{\bm  \varepsilon}(\b y, \{\b D_j\})$ problem. Under the following conditions:
	\begin{enumerate}
		\item[(1)] $\|\b{x}_J \|_\infty \leq B_J$ for some $B_J > 0$, and
		\item[(2)] $\lambda_J <\min \left \{  \frac{1}{2}\left(1 + \frac{1}{\mu(\b D_{[J]})}\right) , \frac{1}{2-2L^t}\left(\frac{1}{\tilde{\mu}(\b D_{[J]})} - 2L^t d_{J}\right) \right \}$ (where $\tilde{\mu}$ is defined in Definition~\ref{def:generalize_coherence}),
	\end{enumerate}
	there exists a CNN activated by $\rho$ with 
		 kernel size $s $,
		 depth $O\left(tK \log_s \prod_{j=1}^{J+1}(d_{j-1}+d_{j})\right)$, and
		 number of weights $O\left( tK \sum_{j=1}^{J+1}(d_{j-1}+d_j)^2 \right)$
	such that its output sequence $\{\tilde{\b x}_j\}_{j=1}^J$ satisfies:
	\begin{align}\label{eq:gen_bound}
		\|\tilde{\b x}_j - \b x_j \|_2 \leq  C_{1} e^{-c K} + C_{2}\left(L^t \beta + \|\bm \varepsilon\|_\infty\right),
		\quad \forall j \in [J],
	\end{align}
	where constants $c, C_{1}, C_{2} > 0$ depend solely on $\{\b D_j \}_{j=1}^J$, $\bm \lambda$, $B_J$.
\end{thm}

If $L<1$, employing ReLU-type activation functions introduces an extra term $L^t \beta$ to \eqref{eq:gen_bound}, which decreases exponentially in $t$, whereas the necessary network depth and number of weights grow only linearly in $t$.
Compared to Theorem \ref{l1DSC}, Theorem \ref{thm:relutype} imposes an extra sparsity condition:
\begin{align*}
	\lambda_J < \frac{1}{2-2L^t}\left(\frac{1}{\tilde{\mu}(\b D_{[J]})} - 2L^t d_{J}\right),
\end{align*}
For the ReLU case ($L=0$), this simplifies to $\lambda_J < 1/\left(2\tilde{\mu}(\b D_{[J]})\right)$. As shown in Remark \ref{remark:nonlinear}, our proof technique allows relaxation to the weaker condition $\lambda_J < \frac{1}{2}\left(1 + 1/\tilde{\mu}(\b D_{[J]})\right)$ for ReLU networks. This relaxed bound remains compatible with Theorem \ref{l1DSC} since 
\[
	\frac{1}{2}\left(1 + \frac{1}{\mu(\b D_{[J]})}\right) 
	\leq \frac{1}{2}\left(1 + \frac{1}{\tilde{\mu}(\b D_{[J]})}\right).
\]
Hence, Theorem \ref{thm:relutype} is consistent with that of Theorem \ref{l1DSC}.
For other activations ($0<L<1$) with $t $ large enough,
the additional sparsity constraint remains practically reasonable for non-ReLU activation $\rho$, as the dominant limiting factor becomes $\frac{1}{2}\left(1/\tilde{\mu}(\b D_j) \right)$.

For activation functions that are not of the ReLU type and for architectures other than CNNs, we can infer their capacity for sparse feature learning from the following result, which involves substituting ReLU with any other neural network that can approximate it closely.

\begin{thm}[Approximation with general activation functions]\label{thm:nonrelu}
	Let $M>0$ and $\rho:\RR \rightarrow \RR$ be an activation function. Assume the existence of a neural network $\phi_\rho$ (with finite neurons) activated by $\rho$ satisfying:
	\begin{align}\label{eq:approx_sigma}
		\|\phi_\rho - \sigma\|_{L_\infty([-M,M])} \leq \delta,
	\end{align}
	where $\delta > 0$ depends on the depth, width or parameters of $\phi_\rho$ and $M $ is sufficiently large. Under the conditions of Theorem \ref{l1DSC}, there exists a neural network $\Phi_\rho$ activated by $\rho$ with its output sequence $\{\tilde{\b x}_j\}_{j=1}^J$ satisfying
	\begin{align}\label{eq:nonrelu_bound}
		\|\tilde{\b x}_j - \b x_j \|_\infty \leq  C_{1} e^{-c K} + C_{2} \|\bm \varepsilon\|_\infty + C_{3} \delta,
	\end{align}
	where constants $c, C_{i} > 0$, $i=1,2,3$ depend solely on $\{\b D_j\}$, $ B_J$, and $\bm \lambda$.
\end{thm}

\begin{remark}
These generalized framework (\thmref{thm:relutype} \& \thmref{thm:nonrelu}) accommodates activations:
\begin{enumerate}
	\item[(1)] LeakyReLU: tends to ReLU as $L \to 0$ which is controlled by its slope for the negative input.
	\item[(2)] PReLU, ELU \cite{clevert2015fast}: automatically handles $L$ and $\beta$ using learnable parameters.
	\item[(3)] SignReLU \cite{lin2018research,li2024signrelu}, rational activation functions \cite{boulle2020rational}: approximate ReLU with exact error bounds specified by number of layers and weights.
\end{enumerate}
Other activation functions such as Swish \cite{ramachandran2017searching,elfwing2018sigmoid} (used in the large language model LLaMA \cite{touvron2023llama}), Mish \cite{misra2019mish} (used in the object detection model YOLOv4 (\cite{bochkovskiy2020yolov4})), which has a similar shape as ReLU, also have a sparse feature learning ability, as shown in \corref{cor:relutype}. These results suggest that many popular activation functions preserve the core learning capacity for sparse coding tasks. A lot of existing activation functions are reviewed in \cite{dubey2022activation} and satisfy our proposed framework.
\end{remark}

\begin{remark}[deep sparse coding via self attention and transformer]
	The attention mechanism is a widely used technique in deep learning, originating from some phenomenon that human beings awareness is likely to concentrating on important information. The scaled dot-product attention \cite{vaswani2017attention} contains three parts: key $\b K$, value $\b V$, and query $\b Q$, and they are linear transforms of the input matrix $\b X$. It was proved that scaled dot-product attention together with positional encoding can approximate any convolutional layers \cite{cordonnier2020relationship}. Moreover, the transformer truck can be briefly formulated as alternatively applying scaled dot product attention and fully connected neural networks \cite{vaswani2017attention}. Hence, combining results in \cite{cordonnier2020relationship}, our results established in this section can be extended to scaled dot-product attention and transformer, implying that deep sparse feature extraction ability exists widely in deep learning architectures.
\end{remark}

\section{Experiments}



In the previous section, we presented various findings that demonstrate that CNNs can effectively function as deep sparse coding solvers, efficiently approximating sparse features. Naturally, this begs the question: Could enforcing sparsity in neural networks further enhance their performance?
To evaluate the influence of sparsity on the performance of neural networks, we performed a series of experiments on image classification and image segmentation tasks. 

{\bf Training strategy.} To obtain sparser features, we introduce an $\ell_1$ penalty term into the loss function. This penalty term acts as a regularization term, encouraging the network to learn sparse representations. Specifically, we modify the original loss function $L$ as follows:
\begin{align*}
	L_{\text{sparse}} : = L + \gamma \sum_j \omega_j \|\b x_j\|_1,
\end{align*}
where $\b x_j$ are the features of the output of some intermediate layers, $\gamma$ is the trade-off hyperparameter, and $\omega_j$ is used to balance the sparsity of features from different layers. We will define these selections separately for each experiment.



\subsection{Image Classification}

%

Image classification tasks aim to classify a given image according to a predefined set of labels. In this study, we evaluated the performance of deep learning models on the CIFAR-10 dataset, a widely used benchmark for image classification. The CIFAR-10 data set comprises 50,000 training images and 10,000 test images, covering 10 distinct classes.

For neural network architectures, we select two well-known image classification networks: LeNet-5 \cite{Lecun1998Gradient} and VGG11 \cite{SimonyanKaren2015VDCN}. LeNet-5, a pioneering convolutional neural network, contains approximately 60,000 trainable parameters, while VGG11, a deeper and more complex model, boasts more than 9 million trainable parameters. These networks typically employ convolutional layers to extract discriminative features from images, followed by fully connected layers for classification. The details of the structures can be found in Table~\ref{tab:structure}. In their structures, we only consider to add the $\ell_1$ penalty of features far away from the classifiers to avoid a lot of impact on the prediction part.

To encourage sparsity in the extracted features, we impose an $\ell_1$ penalty on the outputs of the initial convolutional layers. As demonstrated in the previous section, several activation functions, including ReLU-type functions and others, possess the ability to learn sparse feature representations in neural networks. Therefore, we investigate the performance of three distinct activation functions: ELU \cite{clevert2015fast}, Mish \cite{Misra2019MishAS}, and ReLU.

\begin{table}[!ht]
	\centering
	\begin{tabularx}{\textwidth}{ll}
		\hline
		LeNet-5 & {\color{red} conv5-6 } + maxpool + conv5-16 + maxpool + FC120 + FC84 + FC10 \\ \hline
		VGG11 &  \makecell[l]{{\color{red} conv3-64} +maxpool+ {\color{red} conv3-128}+maxpool + {\color{red} conv3-256 $\times$ 2} + maxpool + conv3-512 $\times$ 2 + maxpool \\+ conv3-512 $\times$ 2 + maxpool + avgpool + FC10}  \\ \hline
	\end{tabularx}
	\caption{Architectures of LeNet-5 and VGG11. The notation `conv5-6' means for this convolution we set kernel size $5$ and output channel $6$ and `FC120' means output dimension is $120$. In each convolution of VGG11, we set padding to be $1$ and employ batch normalization, and for LeNet-5 we set paddings to be $0$. The {\color{red} red} color means the outputs of corresponding blocks are considered in the $\ell_1$ penalty. }
	\label{tab:structure}
\end{table}

\begin{table}[!ht]
	\centering
	\begin{tabular}{|l|l|l|l|}
		\hline
		$\gamma$ & ReLU & ELU &  Mish \\ \hline
		LeNet-5 & $10^{-3}$ & $10^{-2}$ & $10^{-4}$ \\ 
		VGG11 & $10^{-3}$ & $10^{-5}$ & $10^{-3}$ \\ \hline
	\end{tabular}
	\caption{The choice of $\gamma$.}
	\label{tab:gamma}
\end{table}


During training, we use the stochastic gradient descent (SGD) algorithm with a minibatch size of $128$. The learning rate is set using the cosine annealing schedule, starting with an initial value of $0.1$. We adopt the cross-entropy loss function as an objective loss $L$ for optimization. Training is carried out for $200$ epochs. The weight decay is set to $5\times 10^{-4}$. We gather the choice of $\gamma$ in Table~\ref{tab:gamma}.

\begin{figure}[!t] 
	\centering 
	\subfigure[LeNet-5]{\includegraphics[width=0.4\textwidth]
		{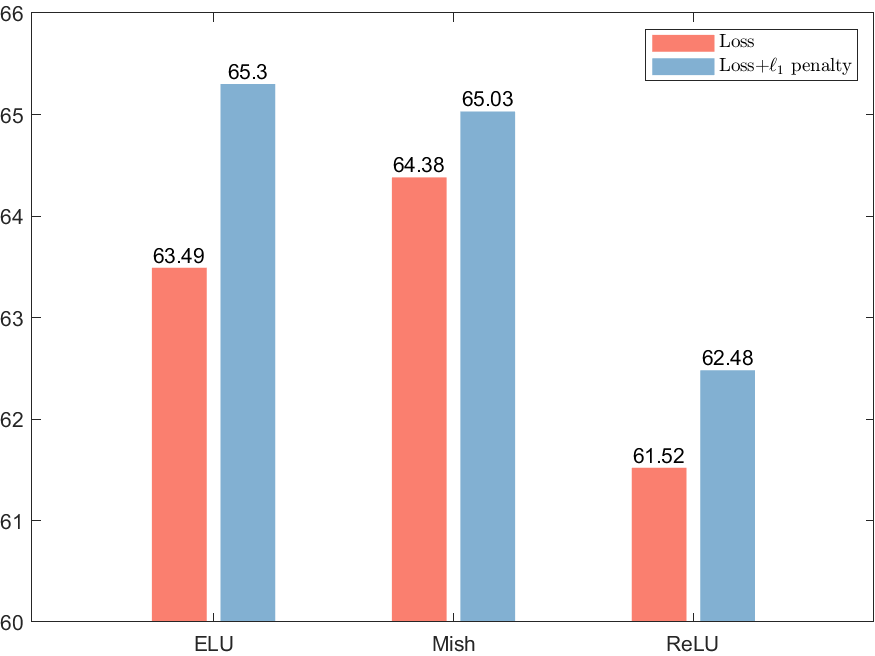}}
	\quad
	\subfigure[VGG11]{\includegraphics[width=0.4\textwidth]
		{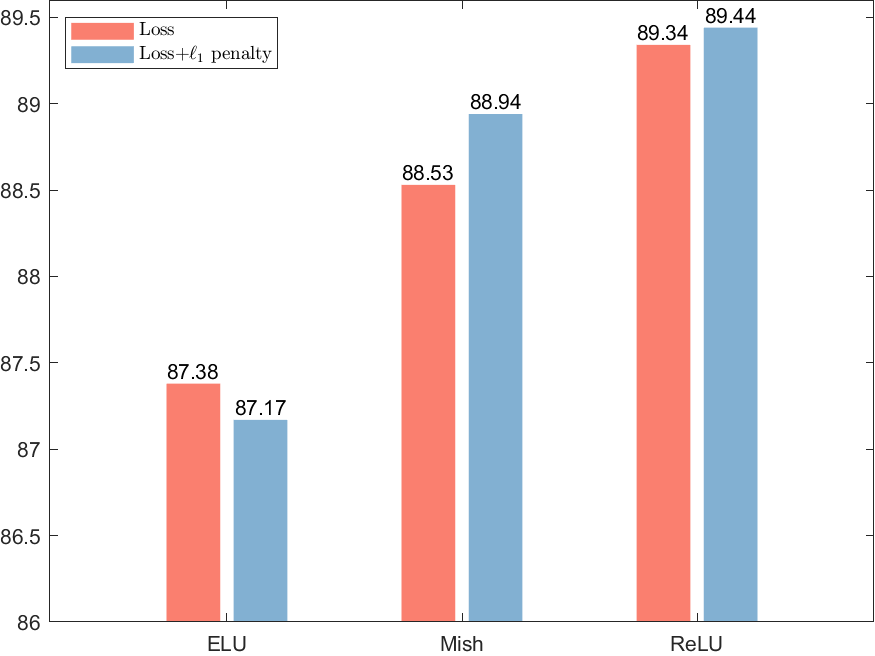}}
	\caption{Test accuracy over CIFAR10.} 
	\label{fig:classification_acc} 
\end{figure}

The test accuracy of each model is collected in Figure~\ref{fig:classification_acc}. We can see that in most cases the $\ell_1$ penalty helps to improve the classification accuracy. Besides, the performance of these activation functions is all within a small range, which indicates that neural networks may work as a deep sparse coding solver no matter what activation functions are used. We also collect the training loss in Figure~\ref{fig:classification_loss}, and plot first $50$ training loss in Figure~\ref{fig:classification_loss_50}. These figures show that for the same activation function, training losses have a similar shape, which is more clear in the VGG11 architecture. In Figure~\ref{fig:classification_loss_50}, we can see that for the same activation function, the losses decay very fast to a similar value. Together with the test accuracy, increasing the sparsity could be a safe choice for improving the performance since training loss is unlikely to change much.

\begin{figure}[!t] 
	\centering 
	\subfigure[LeNet-5]{\includegraphics[width=0.4\textwidth]
		{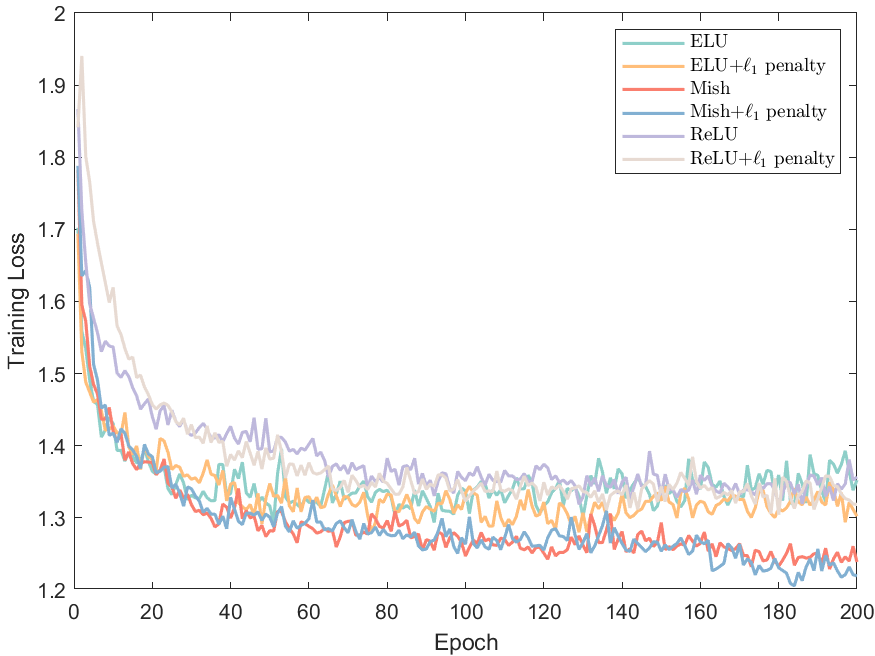}}
	\quad
	\subfigure[VGG11]{\includegraphics[width=0.4\textwidth]
		{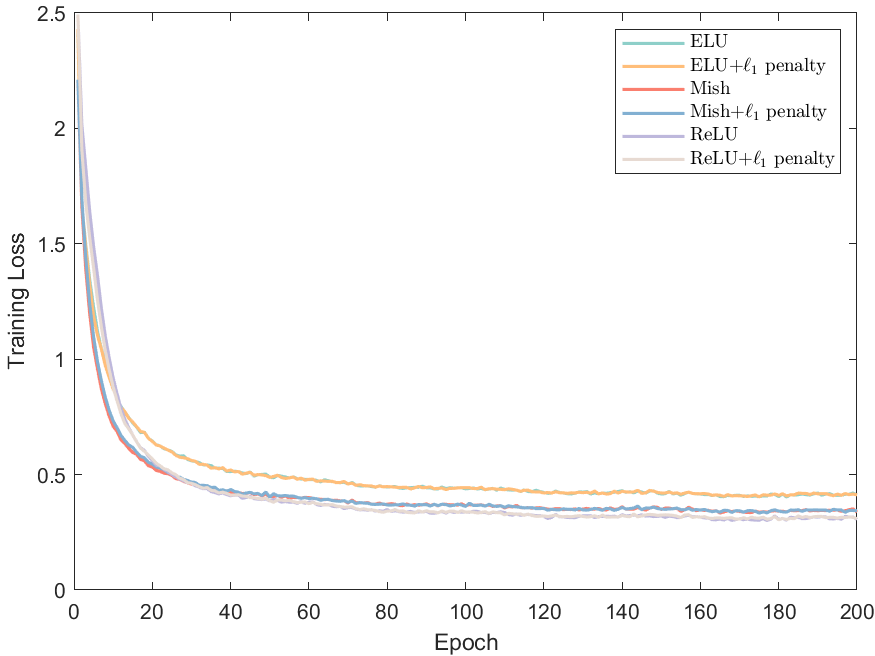}}
	\caption{Training loss over CIFAR10.} 
	\label{fig:classification_loss} 
\end{figure}

\begin{figure}[htbp] 
	\centering 
	\includegraphics[width=0.5\textwidth]
	{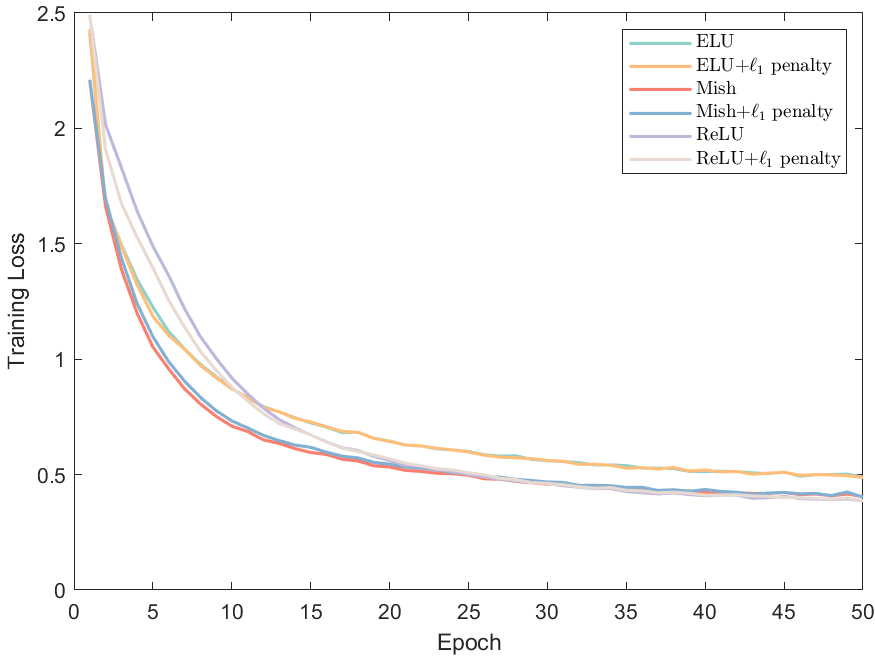}
	\caption{Training loss of VGG11 over first $50$ epoch.}   
	\label{fig:classification_loss_50} 
\end{figure}

\subsection{Image Segmentation}

Image segmentation aims at partitioning pixels from a digital image into multiple segments. Unet \cite{Ronneberger2015Unet} is one of the most well-known deep network architectures to predict pixel-wise labels for image segmentation. The proposed encoding and decoding approach is often referred to as the U structure. Specifically, there are five encoding blocks and four decoding blocks in Unet. After it was noticed by researchers in this field, a lot of work is proposed, designing similar U structures and making neural networks great success in various fields, including computer sciences, medical imaging, and many others.


In this part, we consider $2$ benchmark datasets for pixel-wise binary classification: DUTS \cite{wang2017learning}, DUT-OMRON \cite{wang2013saliency}. DUTS has DUTS-TR for training and DUTS-TE for testing, containing 10553 and 5019 images, respectively. DUT-OMRON comprises of 5168 images for salient objective detection task. Our experiment utilizes DUTS-TR to train Unet and take others for evaluation.

During training, we use the ADAM algorithm with a mini-batch size of $4$. The learning rate is set using the cosine annealing schedule, starting with an initial value of $0.005$. We adopt the binary cross-entropy\footnote{https://pytorch.org/docs/stable/generated/torch.nn.BCELoss.html} (BCE) loss function as the objective loss $L$ for optimization. The training is conducted for $200$ epochs and we set $\gamma$ to be $10^{-3}$. With respect to the penalty term in $L_{\text{sparse}}$, we push the features from the first four encoding blocks to have smaller $\ell_1$ norm with $\omega_j := 1/\text{dim}(\b x_j)$.

To evaluate the performance of the prediction against the ground truth, we employ the following three measurements: 
\begin{align*}
	\text{PA}&:= \f{ \text{TP} + \text{TN} }{\text{TP} + \text{TN} + \text{FP} + \text{FN}}, \\
	\text{mPA}&:= \f{1}{2} \left(\f{ \text{TP}  }{\text{TP} + \text{FP} } + \f{ \text{TN}  }{\text{TN} + \text{FN} } \right), \\
	\text{mIoU}&:= \f{1}{2} \left(\f{ \text{TP} }{\text{TP} + \text{FP} + \text{FN}} + \f{ \text{TN} }{\text{TN} + \text{FP} + \text{FN}} \right). 
\end{align*}
which are based on the definition of the confusion matrix Table~\ref{tab:confusion}.
\begin{table}[]
	\centering
	\begin{tabular}{|c|cc|}
		\hline
		\diagbox{Prediction}{Label}	& Positive & Negative \\  \hline
		Positive	& TP & FP \\ 
		Negative	& FN & TN  \\ \hline
	\end{tabular}
	\caption{Confusion Matrix.}
	\label{tab:confusion}
\end{table}

Table~\ref{tab:seg_mean} collects the average performance of Unet on each dataset. With an additional $\ell_1$ penalty and the keep of other hyperparameters, the performance can be improved with most measurements. Furthermore, the standard diviation in Table~\ref{tab:seg_mean} shows that improving the sparsity could possibly yield a more stable model. Again, in Figure~\ref{fig:segloss} (a), the training loss is not greatly affected by the additional penalty term, which is similar to those observed in image classification experiments. Meanwhile, the penalty term can effectively control the $\ell_1$ norm of learned features, as shown in Figure~\ref{fig:segloss} (b). Regarding a feature $\b x_j$, we take the measurement $\|\mathcal{H}_\alpha(\b x_j)\|_0/\text{dim}(\b x_j)$ to quantify the sparsity, where the hard threshold $\mathcal{H}_\alpha$ is defined as $\mathcal{H}_\alpha(\b x_j)_i := (\b x_j)_i$ if $|(\b x_j)_i|>\alpha$ and zero otherwise. We set $\alpha = 10^{-6}$ and collect sparsity results in Figure~\ref{fig:avg-sparsity}. As the neural network deepens, the features become sparser. Even in Unet, features of deeper encoders have good sparsity. These results indicate that neural networks could benefit from sparsity. We also show three segmentation results for each dataset in Figures~\ref{fig:seg_dust} and \ref{fig:seg_dus}.

\begin{table}[!ht]
	\centering 
	\begin{tabular}{|c|c|l|l|l|}
		\hline
		\multicolumn{1}{|l|}{}     & \multicolumn{1}{l|}{} & \multicolumn{1}{c|}{\textbf{PA}} & \multicolumn{1}{c|}{\textbf{mPA}} & \multicolumn{1}{c|}{\textbf{mIoU}} \\ \hline
		\multirow{2}{*}{{\bf DUTS-TE}}   & \textbf{Unet}         & 0.9203$\pm$0.1004                & 0.8423$\pm$0.1529                 & \textbf{0.7584$\pm$0.1840}          \\ \cline{2-5} 
		& \textbf{Unet$+\ell_1$}      & \textbf{0.9224$\pm$0.0972}       & \textbf{0.8446$\pm$0.1494}        & 0.7582$\pm$0.1799                  \\ \hline
		\multirow{2}{*}{{\bf DUT-OMRON}} & \textbf{Unet}         & 0.9166$\pm$0.1097                & 0.8305$\pm$0.1749                 & 0.7503$\pm$0.2060                   \\ \cline{2-5} 
		& \textbf{Unet$+\ell_1$}      & \textbf{0.9209$\pm$0.1021}       & \textbf{0.8322$\pm$0.1724}        & \textbf{0.7524$\pm$0.2011}         \\ \hline
	\end{tabular}
	\caption{Average performance and standard deviation of Unet on each dataset.}
	\label{tab:seg_mean}
\end{table}

\begin{figure}[!t] 
	\centering 
	\subfigure[BCE Loss]{\includegraphics[width=0.4\textwidth]
		{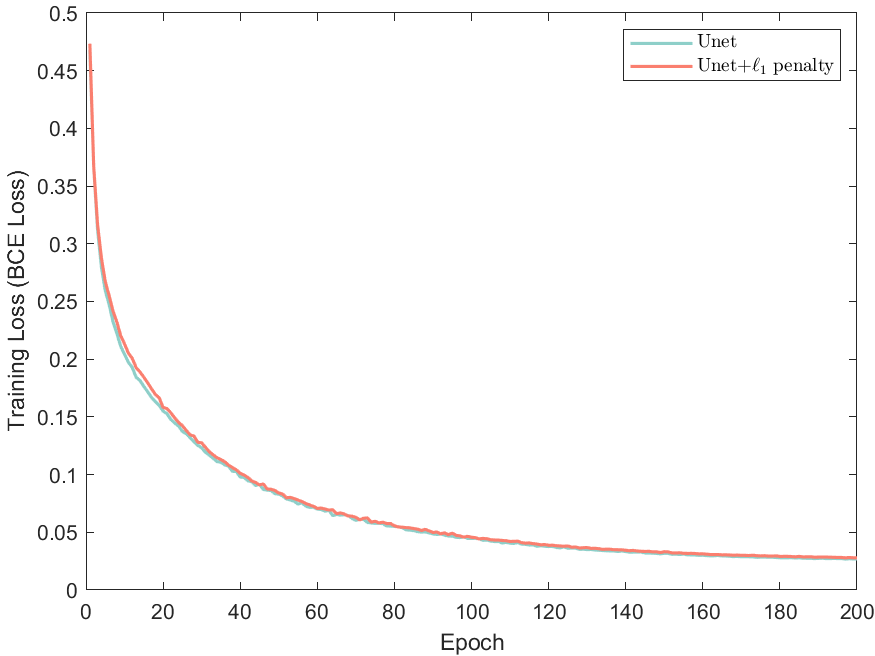}}
	\quad
	\subfigure[$\ell_1$ Loss]{\includegraphics[width=0.4\textwidth]
		{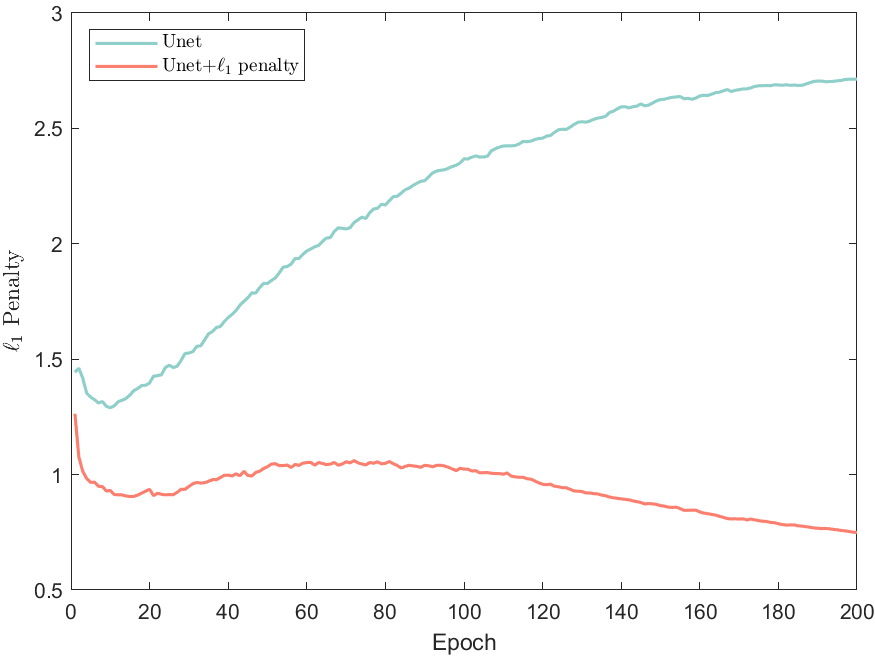}}
	\caption{Training loss and $\ell_1$ penalty ($\sum_j \omega_j \|\b x_j\|_1$) during training.} 
	\label{fig:segloss} 
\end{figure}

\begin{figure}[htbp] 
	\centering 
	\includegraphics[width=0.5\textwidth]
	{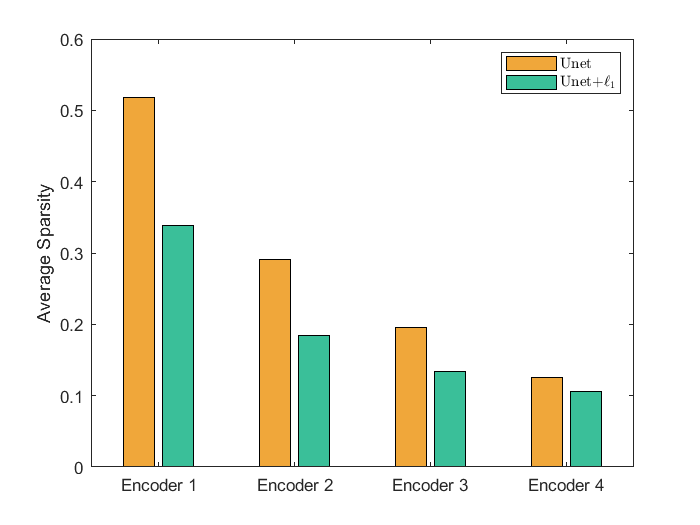}
	\caption{Average sparsity of features produced by Unet over DUT-OMRON.}   
	\label{fig:avg-sparsity} 
\end{figure}

\begin{figure}[!t] 
	\centering 
	\subfigure[Image a]{\includegraphics[height = 2cm, width=2cm]
		{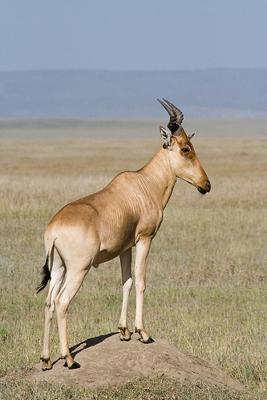}}
	\subfigure[Unet]{\includegraphics[height = 2cm, width=2cm]
		{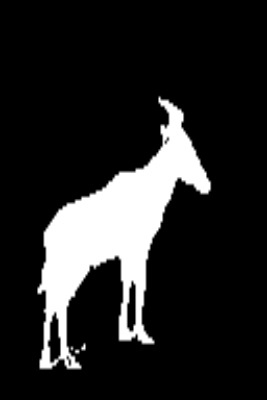}}
	\subfigure[Unet+$\ell_1$]{\includegraphics[height = 2cm, width=2cm]
		{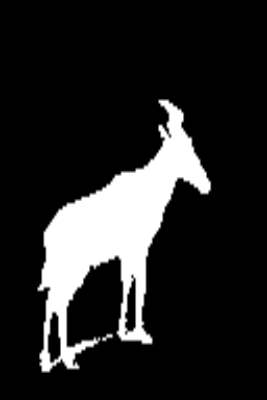}}
	\subfigure[Label]{\includegraphics[height = 2cm, width=2cm]
		{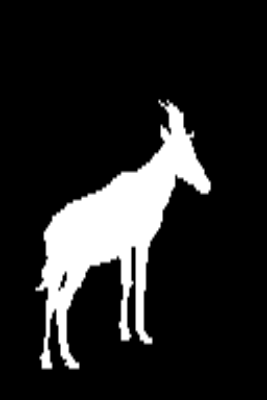}}
	\\
	\subfigure[Image b]{\includegraphics[height = 2cm, width=2cm]
		{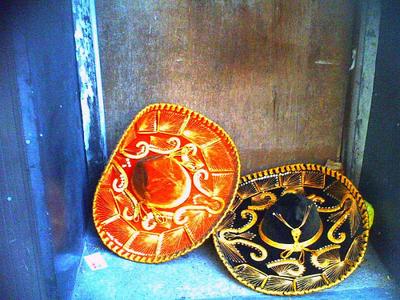}}
	\subfigure[Unet]{\includegraphics[height = 2cm, width=2cm]
		{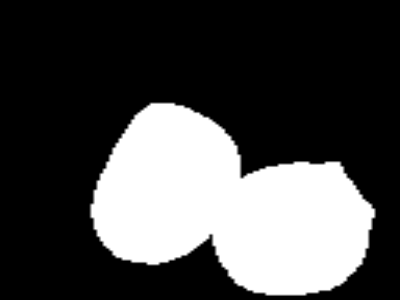}}
	\subfigure[Unet+$\ell_1$]{\includegraphics[height = 2cm, width=2cm]
		{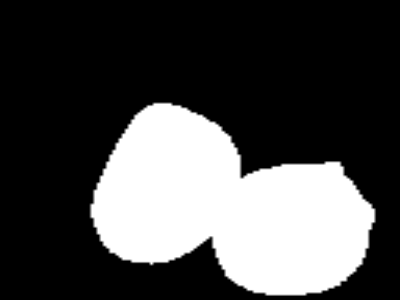}}
	\subfigure[Label]{\includegraphics[height = 2cm, width=2cm]
		{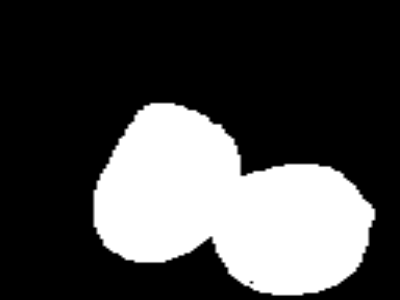}}
	\\
	\subfigure[Image c]{\includegraphics[height = 2cm, width=2cm]
		{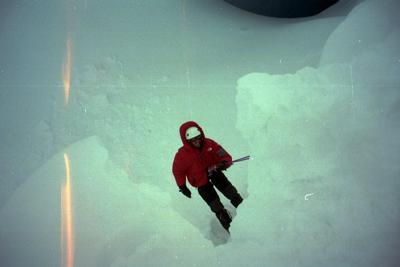}}
	\subfigure[Unet]{\includegraphics[height = 2cm, width=2cm]
		{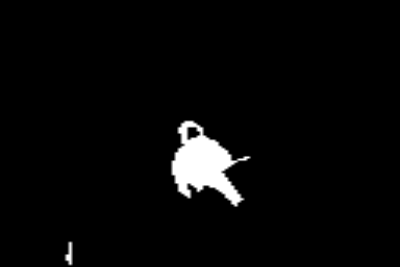}}
	\subfigure[Unet+$\ell_1$]{\includegraphics[height = 2cm, width=2cm]
		{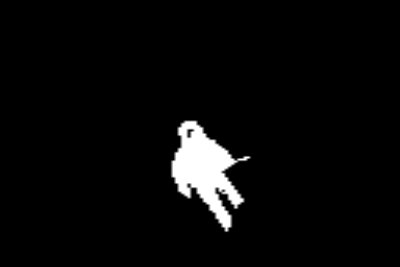}}
	\subfigure[Label]{\includegraphics[height = 2cm, width=2cm]
		{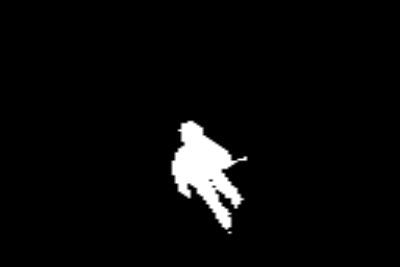}}


	\caption{Three segmentation results over DUTS-TE.} 
	\label{fig:seg_dust} 
\end{figure}

\begin{figure}[!t] 
	\centering 
	\subfigure[Image a]{\includegraphics[height = 2cm, width=2cm]
		{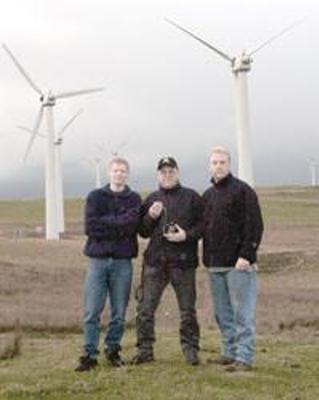}}
	\subfigure[Unet]{\includegraphics[height = 2cm, width=2cm]
		{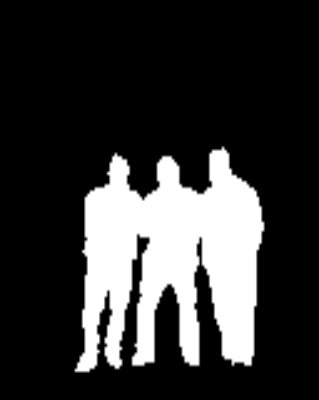}}
	\subfigure[Unet+$\ell_1$]{\includegraphics[height = 2cm, width=2cm]
		{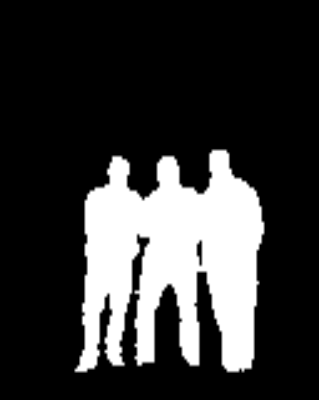}}
	\subfigure[Label]{\includegraphics[height = 2cm, width=2cm]
		{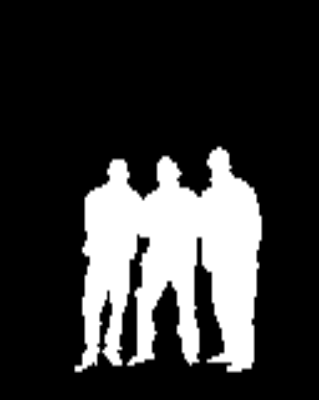}}
	\\
	\subfigure[Image b]{\includegraphics[height = 2cm, width=2cm]
		{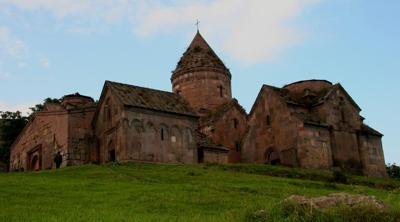}}
	\subfigure[Unet]{\includegraphics[height = 2cm, width=2cm]
		{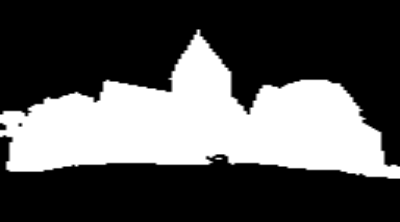}}
	\subfigure[Unet+$\ell_1$]{\includegraphics[height = 2cm, width=2cm]
		{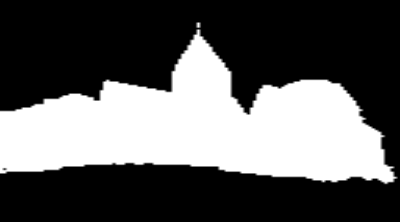}}
	\subfigure[Label]{\includegraphics[height = 2cm, width=2cm]
		{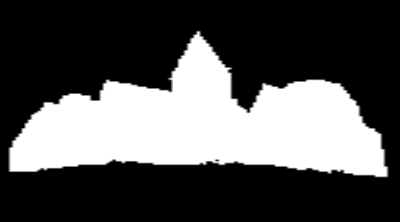}}
	\\
	\subfigure[Image c]{\includegraphics[height = 2cm, width=2cm]
		{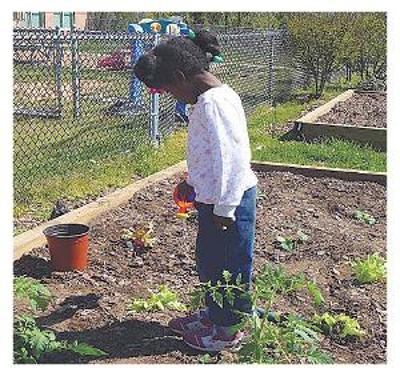}}
	\subfigure[Unet]{\includegraphics[height = 2cm, width=2cm]
		{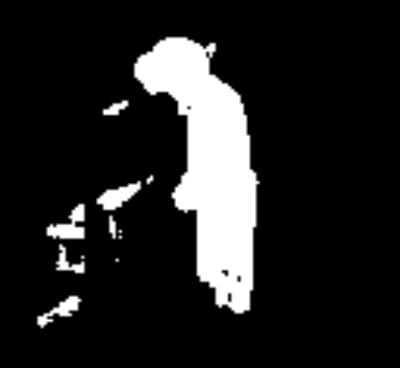}}
	\subfigure[Unet+$\ell_1$]{\includegraphics[height = 2cm, width=2cm]
		{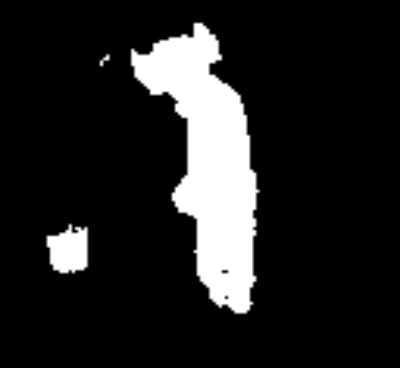}}
	\subfigure[Label]{\includegraphics[height = 2cm, width=2cm]
		{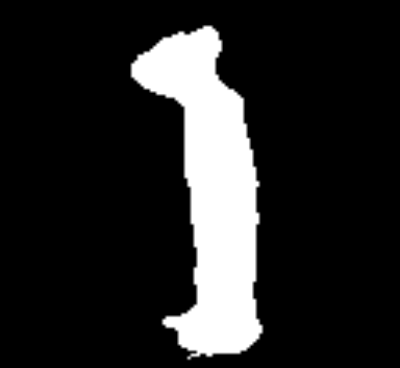}}
	\caption{Three segmentation results over DUT-OMRON.} 
	\label{fig:seg_dus} 
\end{figure}


\section{Conclusion}
In this work, we have studied the ability of convolutional neural networks to learn deep sparse features. We first discussed the uniqueness and stability of deep sparse coding models under a general setting and explored different ways to relax the required conditions. We then provided a convergence analysis for approximating deep sparse features, which offers theoretical support for the observed sparsity in the internal representations of CNNs.
Furthermore, we generalize our results to neural networks with various activation functions and structures, which suggests that deep sparse feature extraction may be a widely occurring phenomenon in many popular neural network architectures.
Motivated by these insights into the sparsity of CNN features, we proposed a $\ell_1$ regularization strategy to train such models. The numerical experiments conducted indicate that the penalty $\ell_1$ can potentially push CNNs to extract sparser features, which in turn can improve the performance in downstream tasks.

\section*{Acknowledgments}
		
The research and the work described in this paper was partially supported  by grants from the Research Grants Council of the Hong Kong Special Administrative Region, China [Projects Nos. CityU 11303821, CityU 11315522, CityU 11308020]. The work of D. X. Zhou was partially supported by Discovery Project (DP240101919) of the Australian Research Council. J. Li gratefully acknowledges support from the project CONFORM, funded by the 
German Federal Ministry of Education and Research (BMBF), as well as from the 
German Research Foundation under the Grant DFG-SPP-2298. Furthermore, J. Li 
acknowledges additional support from the project ``Next Generation AI Computing 
(gAIn)'', funded by the Bavarian Ministry of Science and the Arts and the Saxon 
Ministry for Science, Culture, and Tourism.

\bibliographystyle{plain}
\small
\bibliography{papers1}

\appendix

\section*{Appendix}
This appendix contains detailed theoretical proofs. Appendix~\ref{app:1} presents the proofs of Theorem~\ref{l1DSC}, Theorem~\ref{thm:layer-by-layer}, and the corresponding corollaries. We also provide a detailed discussion in Appendix~\ref{app2} on how to generalize results of ReLU-activated CNNs to CNNs with other activation functions as stated in Theorem~\ref{thm:relutype} and Theorem~\ref{thm:nonrelu}.

\section{Theoretical results of ReLU-activated CNNs for solving deep sparse coding}\label{app:1}
In the following, we first examine the case $J=1$, which corresponds to the classical sparse coding problem.
Let the linear inverse problem $\b y=\b D\b x^* + \bm\varepsilon$ with a given dictionary $\b D \in \RR^{m\times d}$. The signal $\b x^*$ and the observation noise $\bm\varepsilon$ are assumed to be sampled from
\begin{align*}
	(\b x^*, \bm\varepsilon) \in \mathbb{X}(B,\lambda,\delta) := \{ (\b x^*, \bm\varepsilon): \|\b x^*\|_{\infty}\leq B, \|\b x^*\|_0 
	\leq \lambda, \|\bm\varepsilon \|_1 \leq \delta  \}.
\end{align*}

A LISTA-CP sequence $\{ \b x^{(k)} \}_{k=0}^\infty$ is defined to solve the linear inverse problem $\b y = \b D \b x^* +\bm\varepsilon$.
\begin{align}\label{LISTA-CP}
	\begin{aligned}
		\b x^{(0)} &:= 0, \\
		\b x^{(k+1)} &:= \t_{\theta^{(k)}}(\b x^{(k)} + (\b W^{(k)})^\top (\b y - \b D\b x^{(k)})), k = 1,2,3,\dots,
	\end{aligned}
\end{align}
where the soft thresholding function $\t_{\alpha}$ is defined component-wise as
\begin{align*}
	\t_\alpha(\b x)_i = \text{sign}(x_i)(|x_i|-\alpha)_{+}.
\end{align*}
We may also write $\b x^{(k)}$ as $\b x^{(k)}(\b x^*, \bm\varepsilon)$ when we emphasize the dependence. To characterize the condition of $\{ (\b W^{(k)}, \theta^{(k)}) \}_{k=0}^\infty$ that is required for the convergence analysis, we need the concept of generalized mutual coherence.

\begin{defn}\label{def:generalize_coherence}
	The generalized mutual coherence of $\b A \in \RR^{m\times d}$, with each column being normalized, is given by
	\begin{align*}
		\tilde{\mu}(\b A) = \inf_{\b W\in \RR^{m\times d}, \b w_i^\top \b a_i = 1, i\in[d]} \left\{ \max_{i\neq j, i,j\in [d]} \left|\b w_i^\top \b a_j\right| \right\}.
	\end{align*}
\end{defn}
Based on the generalized mutual coherence, we will choose $\b W^{(k)}$ in \eqref{LISTA-CP} from the following set.
\begin{defn}
	Given $\b A$, define the following set
	\begin{align*}
		\mathcal{X}(\b A) = {\operatorname{argmin}}_{\b W \in \RR^{ m\times d}} \left \{ \max_{ij} |W_{ij}|: \b w_i^\top \b a_i = 1, i\in[d], \max_{i\neq j} |\b w_i^\top \b a_j|=\tilde{\mu}(\b A)   \right \}.
	\end{align*}
\end{defn}
The existence of $\b W$ that achieves the value $\tilde{\mu}(\b A)$ is guaranteed by the fact that $\b A$ belongs to the feasible set and the definition of $\tilde{\mu}(\b A)$ can be seen as a linear programming problem \cite{chen2018theoretical}. Now we are ready to introduce the theoretical result of LISTA-CP iteration for solving linear inverse problems.

\begin{thm}[Theorem 2 \cite{chen2018theoretical}]\label{chen18}
	Consider a linear inverse problem $\b y = \b D \b x^* + \bm\varepsilon$ where the columns of $\b D$ are normalized and $(\b x^*, \bm\varepsilon) \in \mathbb{X}(B,\lambda,\delta) $. If each $(\b W^{(k)}, \theta^{(k)})$ satisfies
	\begin{align}\label{conts}
		\b W^{(k)} \in \mathcal{X}(\b D), \quad \theta^{(k)} = \sup_{(\b x^*, \bm\varepsilon) \in \mathbb{X}(B,\lambda,\delta) } \left \{ \tilde{\mu}(\b D) \| \b x^{(k)}(\b x^*,\bm\varepsilon) - \b x^* \|_1  \right\}  + C_{\b W^{(k)}} \delta,
	\end{align}
	where $C_{\b W^{(k)}} = \max_{i\in [m], j\in [d]} |W^{(k)}_{ij}|$, then the sequence $\{\b x^{(k)} \}_{k=0}^\infty$ generated by \eqref{LISTA-CP} for any $(\b x^*, \bm\varepsilon) \in \mathbb{X}(B, \lambda, \delta)$ has the same support as $\b x^*$, i.e. $\supp  \b x^{(k)} = \supp  \b x^*$ and satisfies
	\begin{align*}
		\|\b x^{(k)}(\b x^*,\bm\varepsilon) - \b x^*\|_2 \leq \lambda B e^{-ck} + C \delta, \quad k=1,2,\dots
	\end{align*}
	where $c>0$, $C>0$ are some constants that depend only on $\b D$ and $\lambda$.
\end{thm}

Notice that once the space $\mathbb{X}(B,\lambda,\delta)$ considered is fixed, then we can also fix the choice of $(\b W^{(k)}, \theta^{(k)})$, $k=1,2,\dots$ which is independent of the signal $\b x^*$ and noise $\bm\varepsilon$. Under this observation, we can show that there exists a CNN that can also solve the linear inverse problem.

\begin{lem}[CNNs realize LISTA-CP sequences]\label{cnn-lista}
	Let $K, \lambda \in \NN$, $B, \delta>0$, and $\b D \in \RR^{m \times d}$. There exists a CNN $\Phi$ with kernel size $s$, depth $O\left(K\log_s(d+m)\right)$ and number of weights $O\left( K(m+d)^2 \right)$ such that for any $(\b x^*, \bm\varepsilon) \in \mathbb{X}(B,\lambda,\delta)$, we have $\Phi(\b y) = \b x^{(K)}(\b x^*, \bm \varepsilon)$, where $\b y = \b D \b x^* + \bm \varepsilon$.
\end{lem}
\begin{proof}
	Denote $\b I_d \in \RR^{d\times d}$, $\b O_{m \times d} \in \RR^{m \times d}$ the identity matrix and zero matrix, respectively. We briefly denote $\b O_d := \b O_{d\times d}$ and $\b 0_d := \b O_{d\times1}$ for simplicity. We shall use $\b 1_{d} := (1)_{i\in [d]}$ to represent constant $1$ vectors.
	
	First, the shrinkage operator $\t_\alpha$ can be expressed by ReLU $\sigma(x) = \max \{x,0\}$ as
	\begin{align*}
		\begin{aligned}
			\t_{\alpha}(\b x) &= \sigma(\b x-\alpha \b 1_d) - \sigma(-\b x - \alpha \b 1_d) \\
			&=\left( \b I_d, -\b I_d \right)\sigma \left(
			\begin{pmatrix}
				\b I_d  \\
				-\b I_d
			\end{pmatrix}
			\b x -\alpha
			\begin{pmatrix}
				 \b 1_d \\
				\b 1_d
			\end{pmatrix}
			\right)
			, \quad\forall \b x \in \RR^d.
		\end{aligned}
	\end{align*}
	
	Then the definition \eqref{LISTA-CP} of $\b x^{(k+1)}$ implies
	\begin{align*}
		\begin{aligned}
			\b x^{(k+1)}
			&=
			\t_{\theta^{(k)}} \left (
			\begin{pmatrix}
				\b I_d-  (\b W^{(k)})^\top \b D & (\b W^{(k)})^\top
			\end{pmatrix}
			\begin{pmatrix}
				\b x^{(k)} \\
				\b y
			\end{pmatrix}
			\right) \\
			&= \left( \b I_d, -\b I_d \right)\sigma \left(
			\begin{pmatrix}
				\b I_d  \\
				-\b I_d
			\end{pmatrix}
			\begin{pmatrix}
				\b I_d-  (\b W^{(k)})^\top \b D & (\b W^{(k)})^\top
			\end{pmatrix}
			\begin{pmatrix}
				\b x^{(k)} \\
				\b y
			\end{pmatrix}
			-\theta^{(k)}
			\begin{pmatrix}
				\b 1_d \\
				\b 1_d
			\end{pmatrix}
			\right) \\
			&=
			\left( \b I_d, -\b I_d \right)\sigma \left(
			\begin{pmatrix}
				\b I_d-  (\b W^{(k)})^\top \b D & (\b W^{(k)})^\top  \\
				-\b I_d+  (\b W^{(k)})^\top \b D & -(\b W^{(k)})^\top
			\end{pmatrix}
			\begin{pmatrix}
				\b x^{(k)} \\
				\b y
			\end{pmatrix}
			-\theta^{(k)}
			\begin{pmatrix}
				\b 1_d \\
				\b 1_d
			\end{pmatrix}
			\right).
		\end{aligned}
	\end{align*}
	Thus, we have for $k>1$,
	\begin{align*}
		\begin{aligned}
			\begin{pmatrix}
				\b x^{(k+1)} \\
				\b y
			\end{pmatrix}
			&=
			\begin{pmatrix}
				\b I_d &-\b I_d & \b O_{d\times m} \\
				\b O_{m \times d} & \b O_{m\times d} & \b I_m
			\end{pmatrix}	
			\left(
			\sigma \left(
			\begin{pmatrix}
				\b I_d-  (\b W^{(k)})^\top \b D & (\b W^{(k)})^\top  \\
				-\b I_d+  (\b W^{(k)})^\top \b D & -(\b W^{(k)})^\top \\
				\b O_{m\times d} & \b I_m
			\end{pmatrix}
			\begin{pmatrix}
				\b x^{(k)} \\
				\b y
			\end{pmatrix}
			-
			\begin{pmatrix}
				\theta^{(k)} \b 1_d \\
				\theta^{(k)} \b 1_d \\
				-M \b 1_m
			\end{pmatrix}
			\right)  +
			\begin{pmatrix}
				\b 0_d \\
				\b 0_d \\
				-M \b 1_m
			\end{pmatrix}
			\right)
			\\
			&:= \b A_k \sigma \left( \b B_k
			\begin{pmatrix}
				\b x^{k} \\
				\b y
			\end{pmatrix}
			- \b c_k \right) + \b e_k,
		\end{aligned}
	\end{align*}
	where $\b A_k \in \RR^{(d+m) \times (2d+m) }$, $\b B_k \in \RR^{ (2d+m) \times (d+m)}$, $\b c_k \in \RR^{2d+m}$, $\b e^{(k)} \in \RR^{d+m}$, and $M$ is the minimum infinity norm of all the observation signals $$M := \inf \{C\in \RR: \|\b y\|_\infty \leq C,  \b y =\b  D\b x^*+\bm\varepsilon, (\b x^*, \bm\varepsilon) \in \mathbb{X}(B, s, \delta) \}.$$
	
	For $k=1$, we have $\b x^{(1)} = \t_{\theta^{(0)}} \left( (\b W^{(0)})^\top \b y \right) $ and hence
	\begin{align*}
		\begin{aligned}
			\begin{pmatrix}
				\b x^{(1)} \\
				\b y
			\end{pmatrix}
			&=
			\begin{pmatrix}
				\b I_d &-\b I_d & \b O_{d\times m} \\
				\b O_{m \times d} & \b O_{m \times d} & \b I_m
			\end{pmatrix}	
			\left(
			\sigma \left(
			\begin{pmatrix}
				(\b W^{(0)})^\top   \\
				- (\b W^{(0)})^\top  \\
				\b I_{m}
			\end{pmatrix}
			\b y
			-
			\begin{pmatrix}
				\theta^{(0)} \b 1_d \\
				\theta^{(0)} \b 1_d \\
				-M \b 1_m
			\end{pmatrix}
			\right) +
			\begin{pmatrix}
				\b 0_d \\
				\b 0_d \\
				-M \b 1_m
			\end{pmatrix} 
			\right)
			\\
			&:= \b A_1 \sigma \left( \b B_1
			\b y- \b c_1 \right) + \b e_1,
		\end{aligned}
	\end{align*}
	where $\b A_1 \in \RR^{(d+m) \times (2d+m) }$, $\b B_1 \in \RR^{ (2d+m) \times m}$, $\b c_1\in \RR^{2d+m}$, and $\b e_{1} \in \RR^{d+m}$.
	
	Relying on the result from \cite{FLZ}, which shows that any linear transformation can be implemented by CNNs, it is easy to see that each $\b A_k$ can be realized by a CNN with kernel size $s$, depth $O\left(\log_s (d+m)\right)$ and number of weigths $O((m+d)^2)$ and it is similar for $\b B_{k}$. Concatenating these CNNs together, we obtain a CNN $\Phi$ with depth $O\left(K\log_s(d+m)\right)$ and number of weights $O\left( K(m+d)^2 \right)$ such that $\Phi(\b y) = \b x^{(K)}$.
\end{proof}

Taken together, \thmref{chen18} and \lemref{cnn-lista} show that CNNs can serve as solvers for linear inverse problems. Building on these findings, we now proceed to prove \thmref{l1DSC}.

\begin{proof}[Proof of \thmref{l1DSC}]
The definition of $(DSC_{ 0,\bm\lambda}^{\bm \varepsilon})$ problem implies that there exists a set of vectors $\{ \bm\xi_j \in \RR^{d_{j-1}} \}_{j=1}^J$ such that
	\begin{align*}
		\b y &= \b D_1 \b x_1 + \bm\xi_1, \quad \|\bm \xi_1\|_2 \leq \varepsilon_1, \\
		\b x_1 &= \b D_2 \b x_2 + \bm\xi_2, \quad \|\bm \xi_2\|_2 \leq \varepsilon_2, \\
		&\vdots \\
		\b x_{J-1} &= \b D_J \b x_J + \bm\xi_J, \quad \|\bm \xi_J\|_2 \leq \varepsilon_J.
	\end{align*}
	
	Notice that it is straightforward to rewrite the observation $\b y$ as follows 
	\begin{align*}
		\b y = \b D_{[J]} \b x_J + \sum_{j=1}^J \left( \prod_{i=1}^{j-1}  \b D_i\right) \bm \xi_j,
	\end{align*}
	where we define $\prod_{i=1}^{0}  \b D_i = \b I$.
    Let $\bm \xi_{[J]}:= \sum_{j=1}^J \left( \prod_{i=1}^{j-1}  \b D_i\right) \bm \xi_j $. Notice that $\|\bm \xi_j\|_1 \leq \sqrt{d_{j-1}} \|\bm \xi_j\|_2 \leq \sqrt{d_{j-1}}\varepsilon_j$. 
    The solution $\b x_J$ and noise $\bm\xi_{[J]}$ belong to the collection 
	$$(\b x_J, \bm \xi_{[J]}) \in \mathbb{X}\left(B_J, \lambda_J, \sum_{j=1}^J \sqrt{d_{j-1}} \varepsilon_j \left \| \prod_{i=1}^{j-1} \b D_i \right\|_1 \right) .$$
	Using \thmref{chen18} and \lemref{cnn-lista}, we obtain a CNN $\phi_J$ with kernel size $s$, depth $O(K\log_s(d_0+d_J))$ and number of weights $O(K(d_0+d_J)^2)$ such that
	\begin{align*}
		\|\phi_J(\b y) - \b x_J  \|_2 \leq \lambda_J B_J e^{-c K} + C \sum_{j=1}^J \sqrt{d_{j-1}} \varepsilon_j \left \| \prod_{i=1}^{j-1} \b D_i \right\|_1,
	\end{align*}
	for some positive constant $c, C>0$ and $\supp \phi_J(\b y) = \supp \b x_J$. 
	
	Define $s$-th restricted isometry constant\footnote{The upper bound of $\delta_s $ is given by \cite[Proposition 6.2]{Holger2013book}.} $\delta_s:= \delta_s(\b A)$ of a matrix $\b A$ as the smallest $\delta>0$ satisfying
	\begin{align*}
		(1-\delta)\|\b x\|_2^2 \leq \|\b A \b x\|_2^2 \leq (1+\delta)\|\b x\|_2^2, \quad  \forall \|\b x\|_0 \leq s.
	\end{align*}
	Then we can get the following bound for any $\ell\leq J$
	\begin{align}\label{eq:errordec1}
		\begin{aligned}
			\left \| \prod_{k=\ell}^{J} \b D_{k} \phi_J(\b y) - \prod_{k=\ell}^{J} \b D_{k} \b x_{J} \right \|_2 &\leq \left(1+ \delta_{2\lambda_J} \left(\prod_{k=\ell}^{J} \b D_{k}\right)  \right)^{1/2}  \left( \lambda_J B_J e^{-c K} + C \sum_{j=1}^J \sqrt{d_{j-1}} \varepsilon_j \left \| \prod_{i=1}^{j-1} \b D_i \right\|_1 \right).
		\end{aligned}
	\end{align}
	
	Notice that
	\begin{align}\label{eq:errordec2}
		\begin{aligned}
			\left\|\b x_{\ell-1} - \prod_{k=\ell}^{J} \b D_{k} \b x_{J}  \right\|_2 = \left\| \b D_{[\ell, J]} \b x_{J} + \sum_{k=\ell}^J \left(\prod_{i=\ell}^{k-1} \b D_i \right) \bm \xi_k  - \b D_{[\ell,J]} \b x_{J}  \right\|_2  \leq \sum_{k=\ell}^J 
            \varepsilon_{k} \left\| \prod_{i=\ell}^{k-1} \b D_i \right\|_2.
		\end{aligned}
	\end{align}
	Then combining \eqref{eq:errordec1} and \eqref{eq:errordec2}, we get
	\begin{align*}
		\left \| \prod_{k=\ell}^{J} \b D_{k} \phi_J(\b y) - \b x_{\ell-1} \right \|_2 
		&\leq \left \| \prod_{k=\ell}^{J} \b D_{k} \phi_J(\b y) - \prod_{k=\ell}^{J} \b D_{k} \b x_{J} \right \|_2  + \left\| \prod_{k=\ell}^{J} \b D_{k} \b x_{J}  - \b x_{\ell-1}  \right\|_2 \\
		&\leq C_{1}  e^{-c K} + C_{2} \|\bm \varepsilon\|_\infty,
	\end{align*}
	where $C_1$ and $C_2$ are chosen as
	\begin{align*}
		C_1 &:= \max_{\ell \in [J]}\, \lambda_J B_J \left(1+\delta_{2\lambda_J}\left(\prod_{k=\ell}^{J} \b D_{k}\right) \right)^{1/2} ,\\
        C_2 &:= \max_{\ell \in [J]} \left\{  \sum_{k=\ell}^J  \left\| \prod_{i=\ell}^{k-1} \b D_i \right\|_2, \left(1+  \delta_{2\lambda_J}\left(\prod_{k=\ell}^{J} \b D_{k}\right) \right)^{1/2} C\sum_{j=1}^J \sqrt{d_{j-1}} \left \| \prod_{i=1}^{j-1} \b D_i \right\|_1  \right\}.
	\end{align*}
	
	Using \cite{FLZ} for representing linear transforms into CNNs, there exists a CNN $\psi_j$ with kernel size $s$, depth $O(K\log_s(d_{j-1}+d_j))$ and number of weights $O(K(d_{j-1}+d_j)^2)$ such that $\psi_j(\b z) = \b D_j \b z$. Then we have
	\begin{align*}
		\left \| \psi_{\ell}\circ \cdots \circ \psi_J \circ \phi_J(\b y) - \b x_{\ell-1} \right \|_2 
		\leq C_{1}  e^{-c K} + C_{2} \|\bm \varepsilon\|_\infty.
	\end{align*}
	Setting $\Phi:= \psi_1 \circ \psi_2 \circ \cdots \circ \psi_J \circ \phi_J$, we can conclude that $\Phi$ is a CNN with kernel size $s$, depth $O(K\log_s\prod_{j=1}^{J+1}(d_{j-1}+d_j))$ and number of weights $O(K\sum_{j=1}^{J+1}(d_{j-1}+d_j)^2)$ such that outputs of $\Phi$, denoted as $\{\tilde{\b x}_j \}$, satisfies
	\begin{align*}
			\left \| \tilde{\b x}_j - \b x_{j} \right \|_2 
		\leq C_{1}  e^{-c K} + C_{2} \|\bm \varepsilon\|_\infty, \quad j\in [J].
	\end{align*}
	The proof is complete.
\end{proof}

To show the layer-by-layer strategy, we need the following property of the LISTA-CP sequence.
\begin{lem}[Stability of LISTA-CP sequence]\label{sta-lista}
	Given $\b y , \tilde{\b y} \in  \RR^{m}$, $\b D \in \RR^{m\times d}$, and $\{ (\b W^{(k)}, \theta^{(k)}) \in \mathcal{X}(\b D) \times \RR \}_{k=0}^\infty$, if the sequence $\{ \b x^{(k)} \}_{k=0}^\infty $ is generated by
	\begin{align*}
		\begin{aligned}
			\b x^{(0)} &:= 0, \\
			\b x^{(k+1)} &:= \t_{\theta^{(k)}}(\b x^{(k)} + (\b W^{(k)})^\top (\b y - \b D\b x^{(k)})),\quad k = 1,2,3,\dots,
		\end{aligned}
	\end{align*}
	and the sequence $\{ \tilde{\b x}^{(k)} \}_{k=0}^\infty $ is generated by
	\begin{align*}
		\begin{aligned}
			\tilde{\b x}^{(0)} &:= 0, \\
			\tilde{\b x}^{(k+1)} &:= \t_{\theta^{(k)}}(\tilde{\b x}^{(k)} + (\b W^{(k)})^\top (\tilde{\b y} - \b D\tilde{\b x}^{(k)})),\quad k = 1,2,3,\dots,
		\end{aligned}
	\end{align*}
	then
	\begin{align*}
		\|\b x^{(k)} - \tilde{\b x}^{(k)}\|_2 \leq \left( \sum_{i=0}^{k} d^i \tilde{\mu}(\b  D)^i  \right) \sqrt{md} C_{\b W} \|\b y - \tilde{\b y}\|_2 ,
	\end{align*}
	where $C_{\b W}>0$ is a constant only dependent on $\b D$.
\end{lem}
\begin{proof}
    We denote $C_{\b W} := C_{\b W^{(k)}}$ because the definition of $\X(\b D)$ implies that $C_{\b W^{(k)}} = C_{\b W^{(k')}}$ for any $k\neq k'$.
	Since $\| \t_{\alpha}(\b x) - \t_{\alpha}(\tilde{\b x}) \|_2 \leq \|\b x - \tilde{\b x}\|_2$ for any $\b x, \tilde{\b x} \in \RR^d$, we have
	\begin{align*}
		\|\b x^{(k+1)} - \tilde{\b x}^{(k+1)}\|_2
		&\leq \left \| \left(\b I-(\b W^{(k)})^\top \b D \right)(\b x^{(k)} - \tilde{\b x}^{(k)}) \right \|_2 +\left \|\left(\b W^{(k)}\right)^\top (\b y-\tilde{\b y}) \right \|_2 \\
		&\leq \left \| \b I-(\b W^{(k)})^\top \b D \right \|_F \| \b x^{(k)} - \tilde{\b x}^{(k)}  \|_2 + \| \b W^{(k)} \|_F \|\b y - \tilde{\b y}\|_2 \\
		&\leq d \tilde{\mu}(\b D) \| \b x^{(k)} - \tilde{\b x}^{(k)}  \|_2 + \sqrt{md}C_{\b W} \|\b y - \tilde{\b y}\|_2 \\
		&\leq (d \tilde{\mu}(\b D))^{k+1} \|\b x^{(0)} - \tilde{\b x}^{(0)}\|_2 +  \left( \sum_{i=0}^{k} d^i \tilde{\mu}(\b  D)^i  \right) \sqrt{md} C_{\b W} \|\b y - \tilde{\b y}\|_2 \\
		&= \left( \sum_{i=0}^{k} d^i \tilde{\mu}(\b  D)^i  \right) \sqrt{md} C_{\b W} \|\b y - \tilde{\b y}\|_2 ,
	\end{align*}
	where the second inequality follows from $\|\b A\b x\|_2 \leq \|\b A\|_2 \|\b x\|_2$ and $\|\b A\|_2 \leq \|\b A\|_F$ for any matrix $\b A \in \RR^{m\times d}$ and $\b x \in \RR^d$ and in the last inequality we use the initialization $\b x^{(0)} = \tilde{\b x}^{(0)} = 0$.
\end{proof}


\lemref{sta-lista} is instrumental in quantifying how errors propagate from shallow layers to deeper ones. In the following, we establish the proof of Theorem~\ref{thm:layer-by-layer}.

\begin{proof}[Proof of Theorem~\ref{thm:layer-by-layer}]
    Using an argument similar to the proof of \thmref{l1DSC}, there exists a set of vectors $\{ \bm\xi_j \in \RR^{d_{j-1}} \}_{j=1}^J$ such that
	\begin{align*}
		\b y &= \b D_1 \b x_1 + \bm\xi_1, \quad \|\bm \xi_1\|_1 \leq  \sqrt{d_0} \varepsilon_1 , \\
		\b x_1 &= \b D_2 \b x_2 + \bm\xi_2, \quad \|\bm \xi_2\|_1 \leq \sqrt{d_1}\varepsilon_2, \\
		&\vdots \\
		\b x_{J-1} &= \b D_J \b x_J + \bm\xi_J, \quad \|\bm \xi_J\|_1 \leq \sqrt{d_{J-1}}\varepsilon_J.
	\end{align*}
	
    Let $k_j \in \NN$, $j=1,\dots,J$. Using \lemref{cnn-lista} and \thmref{chen18}, we establish the existence of a CNN $\phi_{j}$ with kernel size $s$, depth $O(k_j\log_s(d_{j-1}+d_j))$ and number of weights $O(k_j(d_{j-1}+d_j)^2)$ to solve the inverse problem $\b x_{j-1} = \b D_{j} \b x_{j} + \bm \xi_j $, $j=1,\dots J$ (where $\b x_0 := \b y$), satisfying
    \begin{align}\label{eq:sta_of_phi}
        \begin{aligned}
            \| \phi_{j}(\b x_{j-1}) - \b x_{j} \|_2 
            &\leq \lambda_{j} B e^{-c_{j}k_j} + C_j \varepsilon_j \sqrt{d_{j-1}} , \\
             \|\phi_{j}(\b z) - \phi_{j}(\b z') \|_2 &\leq \bar{C}_{j} \|\b z - \b z'\|_2 ,\quad \forall z, z' \in \RR^{d_{j-1}},
        \end{aligned}
    \end{align}
    where the second inequality follows from \lemref{sta-lista}, $C_j>0$ depends only on $\lambda_j$ and $\b D_j$, and the constant $c_j, \bar C_j$ are given by
    \begin{align*}
        c_{j} 
        &:= -\ln \left(2\lambda_{j}\tilde{\mu}(\b D_{j})-\tilde{\mu}(\b D_{j}) \right), \\
        \bar C_{j} 
        &:= \left( \sum_{i=0}^{k_j} d_{j}^i \tilde{\mu}(\b  D_{j})^i  \right) \sqrt{d_{j-1} d_j } C_{\b W_{j}}, 
    \end{align*}
    for some constant $C_{\b W_{j}} \geq 0$, where $c_j$ is proved in \cite{chen2018theoretical}.
    
    Define $\phi^{(j)}$ as
    \begin{align*}
        \phi^{(1)} &= \phi_1, \\
        \phi^{(j)} &= \phi_j \circ \phi^{(j-1)}.
    \end{align*}
    Then using \eqref{eq:sta_of_phi}, we obtain for $j>1$,
    \begin{align*}
        \| \phi^{(j)}(\b y) - \b x_j \|_2 
        &\leq \| \phi^{(j)}(\b y) - \phi_{j}(\b x_{j-1}) \|_2 + \| \phi_{j}(\b x_{j-1}) - \b x_{j} \|_2 \\
        &\leq \bar C_{j} \| \phi^{(j-1)}(\b y) - \b x_{j-1} \|_2 +  \lambda_{j} B e^{-c_{j}k_j} + C_j \varepsilon_j \sqrt{d_{j-1}} .
    \end{align*}
    This implies
    \begin{align*}
        \begin{aligned}
            \| \phi^{(j)}(\b y) - \b x_j \|_2 
            &\leq  \bar C_{j} \| \phi^{(j-1)}(\b y) - \b x_{j-1} \|_2 +  \lambda_{j} B e^{-c_{j}k_j} + C_j \varepsilon_j \sqrt{d_{j-1}},  \\
            \bar C_j\cdot \| \phi^{(j-1)}(\b y) - \b x_{j-1} \|_2 
            &\leq \bar C_j \cdot \left(\bar C_{j-1} \| \phi^{(j-2)}(\b y) - \b x_{j-2} \|_2 + \lambda_{j-1} B e^{-c_{j-1}k_{j-1}} + C_{j-1}\varepsilon_{j-1}\sqrt{d_{j-2}} \right), \\
            &\vdots \\
            \prod_{i=2}^j \bar C_i \| \phi^{(1)}(\b y) - \b x_{1} \|_2 
            &\leq \prod_{i=2}^j \bar C_i \left(  \lambda_{1} B e^{-c_{1}k_1} +  C_1 \varepsilon_1 \sqrt{d_0} \right) .
        \end{aligned}
    \end{align*}
    
    Summarizing the above inequalities, we get
    \begin{align}\label{eq:orginal_bound}
        \begin{aligned}
            \| \phi^{(j)}(\b y) - \b x_j \|_2 
            &\leq \sum_{i=1}^j \prod_{m=i+1}^j \bar C_m \left( \lambda_{i} B e^{-c_{i} k_i} + C_i\varepsilon_i\sqrt{d_{i-1}} \right) \\
            &\leq \sum_{i=1}^j \prod_{m=i+1}^j \bar C_m \left( \lambda_{i} B e^{-c k_i} + C_i\varepsilon_i\sqrt{d_{i-1}} \right)
        \end{aligned}
    \end{align}
    where we define $\prod_{m=i+1}^j \bar C_m := 1$ for any $i\geq j$ and $c := \min \{ c_i, i=1,\dots,J\}$.
    
    Since $0 \leq \tilde{\mu}(\b D_j) \leq \mu(\b D_j) \leq 1$ and $\ln d_j \leq C$ for any $j$, the constant $\bar C_j \leq 2 C_{\b W_j} d_j^{k_j} \sqrt{d_{j-1}d_j} $ and hence we derive
    \begin{align}\label{eq:CM}
    \begin{aligned}
        \prod_{m=i+1}^j \bar C_m &\leq \exp\left\{\sum_{m=i+1}^j k_m\ln d_m \right \} \prod_{m=i+1}^j 2 C_{\b W_m} \sqrt{d_{m-1}d_m}  \\
        &\leq \exp\left\{C\sum_{m=i+1}^j k_m \right \} \prod_{m=i+1}^j 2 C_{\b W_m} \sqrt{d_{m-1}d_m}  \\
        &:= \tilde{C}_i \exp\left\{C\sum_{m=i+1}^j k_m \right \} .
    \end{aligned}
    \end{align}
    where we denote $\tilde C_i := \prod_{m=i+1}^j 2 C_{\b W_m} \sqrt{d_{m-1}d_m}$.

Let $a = 1 + 4 \frac{C}{c} $ and define the sequence \( \{k_j\} \) as
$k_j =\lceil a^{J-j}k/c \rceil $ for any integer $k\geq c$. For any \( 1 \leq i < j \leq J \), consider the summation term:  
$$
\sum_{m=i+1}^j k_m \leq 2 \sum_{m=i+1}^j \frac{a^{J-m}k}{c} \leq  \frac{a^{J-i-1}}{c} \cdot \frac{2k}{1 - a^{-1}} = \frac{2 a^{J-i}k}{c(a-1)}.
$$  
Thus the following inequality holds
\begin{align}\label{eq:relax}
\begin{aligned}
    C \sum_{m=i+1}^j k_m - c k_i 
    &\leq \frac{2 C a^{J-i}k}{c(a-1)} - a^{J-i}k \\
    &= a^{J-i} k \left( \frac{2C}{c(a-1)} - 1 \right)  \\
    &\leq -\frac{1}{2} k.
\end{aligned}
\end{align}  

Substituting \eqref{eq:CM} and \eqref{eq:relax} into \eqref{eq:orginal_bound}, we obtain the bound
    \begin{align*}
        \| \phi^{(j)}(\b y) - \b x_j \|_2 
            &\leq \sum_{i=1}^j \left( \tilde C_i \lambda_{i} B e^{-\frac{1}{2} k} + \tilde C_i e^{ck_i - \frac{1}{2} k } C_i \varepsilon_i\sqrt{d_{i-1}} \right) \\
            &\leq \sum_{i=1}^j \left( \tilde C_i \lambda_{i} B e^{-\frac{1}{2} k} + \tilde C_i e^{\left(2 a^{J-i}-\frac{1}{2}\right ) k } C_i \varepsilon_i\sqrt{d_{i-1}} \right).
    \end{align*}
    
    Choosing $C_1^* := \max_{j=1,\dots,J} \sum_{i=1}^j \tilde C_i \lambda_i B$ and $C_2^* := \max_{j=1,\dots,J } \sum_{i=1}^j \tilde C_i C_i \sqrt{d_{i-1}} $, we obtain the bound
    \begin{align*}
        \| \phi^{(j)}(\b y) - \b x_j \|_2 
            \leq  C_1^* e^{-\frac{1}{2} k} + C_2^* \|\bm \varepsilon\|_\infty \sum_{i=1}^j e^{\left (2 a^{J-i}-\frac{1}{2}  \right) k}.
    \end{align*}

	Since all $k_j = O(k)$, it is easy to see that $\Phi:= \phi^{(J)}$ is a CNN with kernel size $s$, depth $O(k\log_s\prod_{j=1}^{J}(d_{j-1}+d_j))$ and number of weights $O(k\sum_{j=1}^{J}(d_{j-1}+d_j)^2)$. The proof is complete.
\end{proof}

The proof of \corref{cor:l2l1} follows from \thmref{thm:layer-by-layer}.

\begin{proof}[Proof of \corref{cor:l2l1}]
	Since $\b x_{j}$ is the unique solution of the sparse linear inverse problem $\b x_{j-1}= \b D_{j}\b x$ with the smallest $\ell_1$ norm, as proven in \cite[Theorem 4.5]{elad2010sparse}, we obtain $L_j(\b x) - L_j(\b x_j) \geq 0$ for any $\b x$ and $L_j(\b x_{j}) = \gamma \|\b x_{j}\|_1$. Then we can derive the following
	\begin{align*}
		L_j(\tilde{\b x}_j) - L_j(\b x_j)
		&\leq \| \b x_{j-1} - \b D_j \tilde{\b x}_j \|^2_2 + \gamma \left( \| \tilde{\b x}_j \|_1 -  \| \b x_j\|_1 \right) \\
		&\leq \| \b D_j\b x_{j} - \b D_j \tilde{\b x}_j \|^2_2 + \gamma \left( \| \tilde{\b x}_j- \b x_j \|_1 \right) \\
		&\leq \|\b D_j\|_2^2 \| \b x_{j} - \tilde{\b x}_j \|_2^2 + \gamma \sqrt{2 \lambda_j } \| \tilde{\b x}_j- \b x_j \|_2 \\
		&\leq C' e^{-2cK} + \gamma \sqrt{2\lambda_j } C_1 e^{-cK} \\
		&\leq C e^{-cK},
	\end{align*}
	for some constant $C>0$, where in the third step we use the fact that $\|\b c\|_1 \leq \sqrt{s} \|\b c\|_2$ for any $\b c\in \RR^s$ and $\| \tilde{\b x}_j \|_0\leq \lambda_j$ which follows from \thmref{chen18} and the proof of \thmref{thm:layer-by-layer}.
\end{proof}

\section{Generalizing ReLU to general activations for deep sparse coding problems}\label{app2}

Before proving \thmref{thm:relutype}, we generalize \thmref{chen18} to the following generalized ReLU-type activation functions, which further extend the concept of ReLU-type activation functions. 

\begin{defn}[Generalized ReLU-type activation functions]\label{thm:moregen_relu}
	An activation function $\rho: \RR \to \RR$ is termed a generalized $(L,\beta)$-ReLU-type activation function if it meets the following criteria:
	\begin{enumerate}
		\item[(1)] Continuity decomposition: $\rho(x) = \sigma(x) + h(x)$, where $ h(x) \in C(\RR)$;
		\item[(2)] Boundedness: $|h(x)| \leq \beta$ for any $x \in \RR$;
		\item[(3)] Lipschitzness: $|h(x) - h(y)| \leq L|x - y|$ for $x,y \in \RR$.
	\end{enumerate}
\end{defn}

The main idea of the proof is inspired by Theorem 2 \cite{chen2018theoretical}.

\begin{thm}\label{lista-nonlinear}
	Let $\rho$ be a generalized ReLU-type activation function satisfying the conditions of Definition \ref{thm:moregen_relu}. Consider a sparse coding problem $\b y = \b D \b{x}^* + \bm\varepsilon$ where columns of $\b D$ are normalized and $(\b x^*, \bm\varepsilon) \in \mathbb{X}(B,\lambda,\delta) $. Let the sequence $\{ \b x^{(k)} \}_{k=0}^\infty$ be generated from the following iteration
	\begin{align*}
		\b x^{(k+1)} = \rho \left( \b x^{(k)} + (\b W^{(k)})^T  (\b y - \b D \b x^{(k)}) - \theta^{(k)} \b 1 \right) - \rho \left( -\b x^{(k)} - (\b W^{(k)})^T (\b y - \b D\b x^{(k)}) - \theta^{(k)} \b 1 \right),
	\end{align*}
	with $\b{x}^{(0)} = \b{0}$ and $\b W^{(k)}$ and $\theta^{(k)}$ defined similarly to \eqref{conts}.

	Then with an additional condition $\lambda < \min \left\{\f{1}{2}\left(1+ \f{1}{ \mu(\b D)} \right),  \f{1}{2-2L}\left( \f{1}{\tilde \mu(\b D)} - 2L d \right)  \right\} $, we have for any given $(\b x^*, \bm\varepsilon) \in \mathbb{X}(B,\lambda,\delta) $ that the following convergence rate holds
	\begin{align*}
		\sup_{(\b x^*, \bm\varepsilon) \in \X(B,\lambda,\delta)}\|\b x^{(k+1)} - \b x^*\|_1
		\leq \lambda Be^{-c(k+1) }   + C(\delta + \beta)  ,
	\end{align*}
	for some positive constants $c, C$.
\end{thm}

\begin{proof}
	
	Let $\rho(x) = \sigma(x) + h(x)$. Firstly, let us rewrite the input of $\rho$ in the iteration as
	\begin{align}\label{eq:non2}
		\begin{aligned}
			 x_i^{(k)} + (\b w^{(k)}_i)^T (\b y -\b  D\b x^{(k)}) &= x_i^{(k)}- (\b w_i^{(k)})^T \b D(\b x^{(k)} -\b  x^*) + (\b w_i^{(k)})^T \bm\varepsilon \\
			&= x_i^{(k)}- \sum_{j=1}^{d} (\b w_i^{(k)})^T\b  d_j(x^{(k)}_j -  x^*_j) + (\b w_i^{(k)})^T \bm\varepsilon \\
			&=  x_i^{(k)}- \sum_{j\neq i} (\b w_i^{(k)})^T \b d_j(x^{(k)}_j - x^*_j) - (x_i^{(k)} - x_i^*) +  (\b w_i^{(k)})^T \bm\varepsilon \\
			&= x_i^*- \sum_{j\neq i} (\b w_i^{(k)})^T \b  d_j(x^{(k)}_j - x^*_j) + (\b w_i^{(k)})^T \bm\varepsilon .
		\end{aligned}
	\end{align}
	Fixing an $i \notin \supp \b x^*$, that is, $x^*_i = 0$, we derive the upper bound of \eqref{eq:non2}
	\begin{align}\label{eq:non1}
		\begin{aligned}
			&\left|x_i^{(k)} + (\b w^{(k)}_i)^T (\b y - \b D\b x^{(k)}) \right|\\
			&= \left|- \sum_{j\neq i} (\b w_i^{(k)})^T \b d_j(x^{(k)}_j - x^*_j) + (\b w_i^{(k)})^T \bm  \varepsilon \right|  \\
			&\leq \tilde{\mu}(\b D) \|\b  x^{(k)} - \b x^* \|_1 + C_{\b W^{(k)}} \delta   \\
			&\leq \theta^{(k)}.
		\end{aligned}
	\end{align}
    This indicates that
        \begin{align}\label{eq:threshold_rho}
            \t_{\theta^{(k)}} \left ( x^{(k)}_i +  (\b w_i^{(k)})^T  (\b y - \b D \b x^{(k)}) \right ) = 0, \quad \forall i \notin \supp \b x^*.
        \end{align}
	For simplicity, let us denote
	\begin{align*}
		H(\b x^{(k)}; \b W^{(k)}, \theta^{(k)})_i = h \left(  x^{(k)}_i + (\b w_i^{(k)})^T  (\b y - \b D\b x^{(k)}) - \theta^{(k)}  \right) - h \left( - x^{(k)}_i - (\b w_i^{(k)})^T (\b y - \b D\b x^{(k)}) - \theta^{(k)}  \right).
	\end{align*}
	Then, employing \eqref{eq:non1} and \eqref{eq:threshold_rho}, we have for any fixed $i \notin \supp \b x^*$
	\begin{align}\label{eq:non4}
		\begin{aligned}
			\left|x^{(k+1)}_i - x_i^* \right| &\leq \left | \sigma \left( x^{(k)}_i + (\b w_i^{(k)})^T  (\b y - \b D\b x^{(k)}) - \theta^{(k)}  \right) - \sigma \left( -x^{(k)}_i - (\b w_i^{(k)})^T (\b y - \b D\b x^{(k)}) - \theta^{(k)}  \right) \right | \\
			&\quad + \left| H(\b x^{(k)}; \b W^{(k)}, \theta^{(k)})_i \right| \\
			&\leq \left | \t_{\theta^{(k)}}\left( x^{(k)}_i + (\b w_i^{(k)})^T  (\b y - \b D\b x^{(k)}) \right) \right| + 2L \theta^{(k)} \\
			&= 2L\theta^{(k)},
		\end{aligned}
	\end{align}
    where in the second step, we use the Lipschitz property of $h(\cdot)$ and \eqref{eq:non1}, and in the last step we use \eqref{eq:threshold_rho}.
	
	Based on \eqref{eq:non2}, we can reformulate $x_i^{(k+1)}$ as
	\begin{align*}
		\begin{aligned}
			x^{(k+1)}_i
			= \t_{\theta^{(k)}}\left( x_i^*- \sum_{j\neq i} (\b w_i^{(k)})^T \b  d_j(x^{(k)}_j - x^*_j) + (\b w_i^{(k)})^T \bm\varepsilon \right) + H(\b x^{(k)}; \b W^{(k)}, \theta^{(k)})_i,
		\end{aligned}
	\end{align*}
	or
    \begin{align}\label{eq:non3}
		\begin{aligned}
			x^{(k+1)}_i -  H(\b x^{(k)}; \b W^{(k)}, \theta^{(k)})_i
			= \t_{\theta^{(k)}}\left( x_i^*- \sum_{j\neq i} (\b w_i^{(k)})^T \b  d_j(x^{(k)}_j - x^*_j) + (\b w_i^{(k)})^T \bm\varepsilon \right) .
		\end{aligned}
	\end{align}
	
	Define the subgradient $\partial \ell_1(x)$ of $|x|$ as
	\begin{align*}
		\partial \ell_1(x)=
		\begin{cases}
			\sign (x), & x \neq 0,\\
			[-1,1], & x = 0.
		\end{cases}
	\end{align*}
	Then \eqref{eq:non3} implies that for any $i$
	\begin{align*}
		x^{(k+1)}_i - H(\b x^{(k)}; \b W^{(k)}, \theta^{(k)})_i  \in  x_i^*- \sum_{j\neq i} (\b w_i^{(k)})^T \b d_j(x^{(k)}_j - x^*_j) + (\b w_i^{(k)})^T\b  \varepsilon - \theta^{(k)}\partial \ell_1\left(x^{(k+1)}_i - H(\b x^{(k)}; \b W^{(k)}, \theta^{(k)})_i\right) ,
	\end{align*}
	or equivalently,
	\begin{align*}
    \begin{aligned}
        x^{(k+1)}_i-x_i^* \in  &- \sum_{j\neq i} (\b w_i^{(k)})^T \b d_j(x^{(k)}_j - x^*_j) + (\b w_i^{(k)})^T \bm\varepsilon - \theta^{(k)}\partial \ell_1\left(x^{(k+1)}_i - H(\b x^{(k)};\b  W^{(k)}, \theta^{(k)})_i\right) \\
        &+ H(\b x^{(k)};\b  W^{(k)}, \theta^{(k)})_i.
    \end{aligned}
	\end{align*}
	This ensures the following upper bound of $x^{(k+1)}_i-x_i^*$
	\begin{align}\label{eq:non5}
    \begin{aligned}
        \left| x^{(k+1)}_i-x_i^* \right|
		&\leq  \left | \sum_{j\neq i} (\b w_i^{(k)})^T \b d_j(x^{(k)}_j - x^*_j) \right |  + C_{\b W^{(k)}} \delta + \theta^{(k)} +  2\beta \\
        &\leq \tilde{\mu}(\b D)  \sum_{j\neq i} \left | x^{(k)}_j - x^*_j \right |  + C_{\b W^{(k)}} \delta + \theta^{(k)} +  2\beta,
    \end{aligned}
	\end{align}
    where we utilize the definition \eqref{conts} of $\b W^{(k)}$ and properties $|h(x)| \leq \beta $ for any $x \in \RR$ and $\|\bm \varepsilon \|_1 \leq \delta$.
    
	Now we can split $\left \|\b x^{(k+1)}- \b x^* \right\|_1$ based on the support of $\b x^*$ by using \eqref{eq:non5} and \eqref{eq:non4} 
	\begin{align}\label{eq:l1upper}
    \begin{aligned}
        \left \|\b x^{(k+1)}- \b x^* \right\|_1
		&= \sum_{i \in \supp \b x^* }\left| x^{(k+1)}_i-x_i^* \right| + \sum_{i \notin \supp \b x^* }\left| x^{(k+1)}_i \right|  \\
		&\leq \tilde{\mu}(\b D) \sum_{i \in \supp \b x^* } \sum_{j\neq i} \left |  x^{(k)}_j - x^*_j \right |  + \lambda C_{\b W^{(k)}} \delta + \lambda\theta^{(k)}+ 2\lambda \beta + 2L(d-\lambda) \theta^{(k)} \\
		&\leq \tilde{\mu}(\b D) \lambda \left \| \b x^{(k)} -\b x^* \right \|_1 + \left(2Ld+ (1-2L)\lambda\right) \theta^{(k)} + \lambda C_{\b W }\delta + 2\lambda\beta.
    \end{aligned}
	\end{align}
    Here we use $C_{\b W} := C_{\b W^{(k)}}$ due to the fact that $C_{\b W^{(k)}} = C_{\b W^{(k')}}$ for any $k\neq k'$.
    
	Substituting the definition \eqref{conts} of $\theta^{(k)}$ into \eqref{eq:l1upper}, we immediately derive
	\begin{align}\label{eq:non6}
		\begin{aligned}
			&\sup_{(\b x^*, \bm\varepsilon) \in \mathbb{X}(B,\lambda,\delta) } \left \|\b x^{(k+1)}- \b x^* \right\|_1 \\
			&\leq \tilde{\mu}(\b D) \left(2Ld + (2-2L)\lambda\right) \sup_{(\b x^*, \bm\varepsilon) \in \mathbb{X}(B,\lambda,\delta) } \left \| \b x^{(k)} -\b x^* \right \|_1 + \left(2Ld + (2-2L)\lambda\right)C_{\b W} \delta + 2\lambda\beta.
		\end{aligned}
	\end{align}
	Denote $\alpha:= \tilde{\mu} (\b D)\left(2Ld + (2-2L)\lambda\right) $. By induction, we obtain
	\begin{align*}
		\begin{aligned}
			\sup_{(\b x^*, \bm\varepsilon) \in \mathbb{X}(B,\lambda,\delta) } \left \|\b x^{(k+1)}- \b x^* \right\|_1 \leq \alpha^{k+1} \sup_{(\b x^*, \bm\varepsilon) \in \mathbb{X}(B,\lambda,\delta) } \left \|\b x^{(0)}- \b x^* \right\|_1 + \left(\alpha C_{\b W} \delta/\tilde\mu(\b D) + 2\lambda\beta \right)  \sum_{i=0}^{k} \alpha^i.
		\end{aligned}
	\end{align*}
	Hence, when $\alpha <1$, that is, $ \lambda  < \f{1}{2-2L}\left( \f{1}{\tilde \mu(\b D)} - 2L d \right)$, we get a constant $c:= -\ln \alpha>0$ satisfying $\alpha^k = e^{-c k}$. We complete the proof by setting $C:= \max\{ \alpha C_{\b W} /\tilde\mu(\b D) , 2 \lambda \} \times \sum_{i=0}^{k} \alpha^i$.
\end{proof}


\begin{remark}\label{remark:nonlinear}
ReLU activation emerges as a special case of the generalized ReLU-type function $\rho$ with parameters $(L, \beta) = (0, 0)$. This specialization makes two difference compared to \thmref{chen18}:

\begin{enumerate}
	\item \textbf{Sparsity Condition}: The mutual coherence constraint simplifies to:
	\[
		\lambda < \frac{1}{2\tilde{\mu}(\b D)},
	\]
	which appears more restrictive than the condition $\lambda < \frac{1}{2}\left(1 + \frac{1}{\tilde{\mu}(\b D)}\right)$ established in \cite[Theorem 2]{chen2018theoretical}. However, this discrepancy can be addressed by modifying the coefficient of $\|\b x^{(k)} - \b x^*\|_1$ in \eqref{eq:l1upper} to $\tilde{\mu}(\b D)(\lambda - 1)$ since in the ReLU case, the upper bound in \eqref{eq:non4} reduces to zero and this implies that iteration recovers the support of the solution $\b x^*$. 
	\item \textbf{Convergence Equivalence}: The resulting condition aligns with \cite[Theorem 2]{chen2018theoretical} when considering the uniqueness requirement:
	\[
		\lambda < \frac{1}{2}\left(\frac{1}{\mu(\b D)} + 1\right) .
	\]
	Since $\mu(\b D) \leq \tilde{\mu}(\b D)$ by definition, our condition remains compatible with the ReLU setting.
\end{enumerate}

These observations suggest that:
\begin{itemize}
	\item Generalized ReLU-type activations maintain comparable convergence properties to ReLU.
	\item The sparsity condition $\lambda < \f{1}{2-2L}\left( \f{1}{\tilde \mu(\b D)} - 2L d \right)$ is not too strong for non-ReLU ones if $L$ is small enough.
\end{itemize}

This theorem underscores the robustness of CNN-based sparse coding frameworks across various activation choices.
\end{remark}

\begin{proof}[Proof of Theorem \ref{thm:relutype}]
The proof strategy extends the framework established in Theorem \ref{l1DSC}, building on Lemma \ref{cnn-lista} and its reliance on the identity-preserving property of ReLU. This identity-preserving property ensures feedforward propagation of inputs through subsequent layers, a critical feature we generalize to arbitrary ReLU-type functions.
 
Define the $t$-fold composition:
\[
	\rho^t(x) := \underbrace{\rho \circ \cdots \circ \rho}_{t \text{ times}}.
\]
This operator inherits three properties:
\begin{itemize}
	\item[(i)] For nonnegative inputs,$\rho^t(x) = x $, $ \forall x \geq 0 $, and $  \rho^t(0) = 0$.
	
	\item[(ii)] For $x,y \leq 0$,
		$|\rho^t(x) - \rho^t(y)| \leq L|\rho^{t-1}(x) - \rho^{t-1}(y)| \leq L^t |x - y| $. This establishes $L^t$-Lipschitz continuity on the negative domain.
 
	\item[(iii)] For $x < 0$, $|\rho^t(x)| \leq L^t \beta$,
this bound quantifies the exponential decay of the negative part magnitude with composition order $t$ if $L<1$.
\end{itemize}
Substituting $\rho^t$ as the activation function in \thmref{lista-nonlinear} produces a CNN activated by $\rho^t$ that approximates the sparse solution. 
Following the steps in \lemref{cnn-lista} and \thmref{l1DSC}, we finish the proof.
\end{proof}

Taking a similar approach, we can derive an analogous result for CNNs with generalized ReLU-type activations. Nevertheless, in this setting, even without noise, the derived error bound does not ensure that the network converges to the unique solution as the number of layers tends to infinity. It only guarantees that the obtained solution remains close to the unique one.

\begin{cor}\label{cor:relutype}
	Let $J \in \NN$, $L, \beta \geq 0$, and $d_0, d_1, \dots, d_{J+1} \in \NN$ with $d_0 = d_{J+1}$. Let $\rho$ be a generalized $(L,\beta)$ -ReLU-type activation function that meets the conditions of Definition \ref{thm:moregen_relu}. Consider a sequence of column-normalized dictionaries $\{\b D_j \in \RR^{d_{j-1} \times d_j} \}_{j=1}^J$.
    For a global observation $\b y$ generated by the $(DSC_{0,\bm\lambda}^{\bm \varepsilon})$ model, let $\{\b x_j\}_{j=1}^J$ be a solution to the corresponding $DSC_{0,\bm \lambda}^{\bm  \varepsilon}(\b y, \{\b D_j\})$ problem. Under the following conditions:
	\begin{enumerate}
		\item[(1)] $\|\b{x}_J \|_\infty \leq B_J$ for some $B_J > 0$, and
		\item[(2)] $\lambda_J <\min \left \{  \frac{1}{2}\left(1 + \frac{1}{\mu(\b D_{[J]})}\right) , \frac{1}{2-2L}\left(\frac{1}{\tilde{\mu}(\b D_{[J]})} - 2L d_{J}\right) \right \}$ (where $\tilde{\mu}$ is defined in Definition~\ref{def:generalize_coherence}),
	\end{enumerate}
	there exists a CNN activated by $\rho$ with 
		 kernel size $s $,
		 depth $O\left(K \log_s \prod_{j=1}^{J+1}(d_{j-1}+d_{j})\right)$, and
		 number of weights $O\left( K \sum_{j=1}^{J+1}(d_{j-1}+d_j)^2 \right)$
	such that its output sequence $\{\tilde{\b x}_j\}_{j=1}^J$ satisfies:
	\begin{align*}
		\|\tilde{\b x}_j - \b x_j \|_2 \leq  C_{1} e^{-c K} + C_{2}\left( \beta + \|\bm \varepsilon\|_\infty\right),
		\quad \forall j \in [J],
	\end{align*}
	where constants $c, C_{1}, C_{2} > 0$ depend solely on $\{\b D_j \}_{j=1}^J$, $\bm \lambda$, $B_J$.
\end{cor}
\begin{proof}
    The proof is the same as that of Theorem \ref{thm:relutype}.
\end{proof}

If the soft-thresholding function or ReLU can be approximated by a neural network $\phi_\rho$, then we can apply the construction from \lemref{cnn-lista}, substitute the nonlinearity with $\phi_\rho$, and obtain a comparable error bound.
\begin{proof}[Proof of \thmref{thm:nonrelu}]
    Consider two iterative sequences generated by different activation functions:
	\begin{align*}
		\b x^{(k+1)}  &= \sigma \left( \b A_1^{(k)} \b x^{(k)} + \b A_2^{(k)} \b y +\b b^{(k)}\right), \\
		\tilde{\b x}^{(k+1)} & = \phi_\rho \left( \b A_1^{(k)} \tilde{\b x}^{(k)} + \b A_2^{(k)} \b y +\b b^{(k)}\right),
	\end{align*}
	with identical initialization $\tilde{\b x}^{(0)} = \b x^{(0)} = \b 0$.
    
	Applying the Lipschitz continuity of $\sigma(\cdot)$ and the approximation bound of $\phi_\rho$ for $\sigma(\cdot)$, we can obtain a one-step error propagation 
	\begin{align*}
		\|\b x^{(k+1)} - \tilde{\b x}^{(k+1)} \|_2 &\leq \left\| \sigma \left( \b A_1^{(k)} \b x^{(k)} + \b A_2^{(k)} \b y +\b b^{(k)}\right) - \sigma \left( \b A_1^{(k)} \tilde{\b x}^{(k)} + \b A_2^{(k)} \b y +\b b^{(k)}\right) \right\|_\infty \\
		&\quad+ \left\| \sigma \left( \b A_1^{(k)} \tilde{\b x}^{(k)} + \b A_2^{(k)} \b y +\b b^{(k)}\right) - \phi_\rho \left( \b A_1^{(k)} \tilde{\b x}^{(k)} + \b A_2^{(k)} \b y +\b b^{(k)}\right) \right\|_\infty  \\
		&\leq \|\b A_1^{(k)}\|_\infty \|\b x^{(k)} - \tilde{\b x}^{(k)} \|_\infty + \delta.
	\end{align*}
	By induction, we derive the cumulative error bound:
	\begin{align}\label{eq:cumulative_error}
    \begin{aligned}
        \|\b x^{(k+1)} - \tilde{\b x}^{(k+1)} \|_\infty
		&\leq \prod_{i=1}^{k} \|\b A_1^{(i)}\|_\infty \|\b x^{(1)} - \tilde{\b x}^{(1)} \|_\infty + \delta \left( 1+ \sum_{i=0}^{k-2} \prod_{j=0}^{i} \|\b A_1^{(k-j)}\|_\infty \right) \\
		&\leq  \delta \left(1+ \sum_{i=0}^{k-1} \prod_{j=0}^{i} \|\b A_1^{(k-j)}\|_\infty\right).
    \end{aligned}
	\end{align}
	For the specific case $\b A_1^{(k)} = \b I - (\b W^{(k)})^\top \b D  $ for some $\b D \in \RR^{m \times d}$, the norm satisfies $\| \b A_1^{(k)} \|_\infty \leq  (d-1)\tilde\mu(\b D)$ by the definition of $\b W^{(k)}$.  Notice that $\|\b x\|_\infty \leq \|\b x\|_2$ for any $\b x \in \RR^d$ and that LISTA-CP \eqref{LISTA-CP} has a similar representation to $\b x^{(k+1)}$. Using the idea of \thmref{l1DSC} and the error bound \eqref{eq:cumulative_error}, we complete the proof.
\end{proof}

\end{document}